\documentclass{article}

    \PassOptionsToPackage{numbers, compress}{natbib}

\usepackage[main, final]{neurips_2026}

\usepackage[utf8]{inputenc} 
\usepackage[T1]{fontenc}    
\usepackage[colorlinks=true, linkcolor=purple, citecolor=green!60!black]{hyperref}       
\usepackage{url}            
\usepackage{booktabs}       
\usepackage{amsfonts}       
\usepackage{nicefrac}       
\usepackage{microtype}      
\usepackage{xcolor}         

\usepackage{tikz}
\usepackage{graphicx}
\usepackage{caption}
\usepackage{subcaption}
\usepackage{pifont}
\usepackage{algorithm}
\usepackage{algorithmic}

\usepackage{amsmath}
\usepackage{amssymb}
\usepackage{mathtools}
\usepackage{amsthm}
\usepackage{thmtools}
\usepackage{thm-restate}

\makeatletter
\@ifundefined{newcounteralias}{}{%
  \renewcommand\thmt@autorefsetup{%
    \@xa\def
    \csname\thmt@envname autorefname\@xa\endcsname
    \@xa{\thmt@thmname}%
  }%
}
\makeatother

\usepackage{tabularx}
\usepackage{enumitem}

\usepackage{multirow}
\usepackage{adjustbox}
\usepackage{wrapfig}

\makeatletter
\newcommand\footnoteref[1]{\protected@xdef\@thefnmark{\ref{#1}}\@footnotemark}
\makeatother

\usepackage[capitalize]{cleveref}
 \AtBeginDocument{%
    \crefname{figure}{Figure}{Figures}%
}
\newcommand{\eqlabelleft}{(}
\newcommand{\eqlabelright}{)}

\creflabelformat{equation}{\eqlabelleft#2#1#3\eqlabelright}

\theoremstyle{plain}

\theoremstyle{definition}

\theoremstyle{remark}

\newcommand{\appropto}{\mathrel{\vcenter{
  \offinterlineskip\halign{\hfil$##$\cr
    \propto\cr\noalign{\kern2pt}\sim\cr\noalign{\kern-2pt}}}}}

\newcommand{\xmark}{\ding{55}}  

\newcommand{\rateinline}[2]{#1 \color{gray}{\scriptsize $\pm$ #2}}

\usepackage{amsmath,amsfonts,bm}

\def\eqref#1{equation~\ref{#1}}

\def\1{\bm{1}}

\def\rvx{{\mathbf{x}}}
\def\rvy{{\mathbf{y}}}

\def\mC{{\bm{C}}}

\def\mK{{\bm{K}}}

\def\mP{{\bm{P}}}

\DeclareMathAlphabet{\mathsfit}{\encodingdefault}{\sfdefault}{m}{sl}
\SetMathAlphabet{\mathsfit}{bold}{\encodingdefault}{\sfdefault}{bx}{n}

\def\emC{{C}}

\def\emP{{P}}

\title{WASD: Wasserstein-based Knowledge Distillation for Large Language Models}

\author{
Byeonghu Na\textsuperscript{\textmd{1}} \quad Donghyeok Shin\textsuperscript{\textmd{1}} \quad Yeongmin Kim\textsuperscript{\textmd{1}} \quad Mina Kang\textsuperscript{\textmd{1}} \quad Il-Chul Moon\textsuperscript{\textmd{1,2}} \\
\textsuperscript{\textmd{1}}KAIST, \textsuperscript{\textmd{2}}summary.ai\\
\texttt{\{byeonghu.na,tlsehdgur0,alsdudrla10,kasong13,icmoon\}@kaist.ac.kr}\\
}

\begin{document}

\maketitle

\begin{abstract}
  Autoregressive large language models (LLMs) have rapidly advanced in capability, but their increasing scale comes with substantial computational and memory costs at inference time. Knowledge distillation (KD) offers a practical solution by transferring knowledge from a large teacher model to a smaller student model via alignment of discrete probability distributions. However, existing KD methods for LLMs primarily rely on divergences that evaluate discrepancies through probability values at each vocabulary index, without explicitly leveraging token-level semantic information. We propose Wasserstein-based knowledge distillation (WASD) for LLMs, which incorporates token-level semantic information via the Wasserstein-based distance with a cost matrix derived from token embeddings. To ensure computational tractability, we adopt the Sinkhorn divergence and derive a gradient-equivalent objective that can be efficiently optimized without introducing additional networks. Experiments across multiple LLM families and scales show that WASD consistently improves distillation performance on diverse tasks, including instruction following, mathematical reasoning, and code generation. Our results highlight the importance of semantic information encoded in the token space for effective distribution alignment in LLM distillation. The implementation is publicly available at \url{https://github.com/aailab-kaist/WASD}.
\end{abstract}

\begin{figure}[t]   
    \centering
    \begin{subfigure}[b]{0.71\linewidth}
        \centering
        \includegraphics[width=\linewidth]{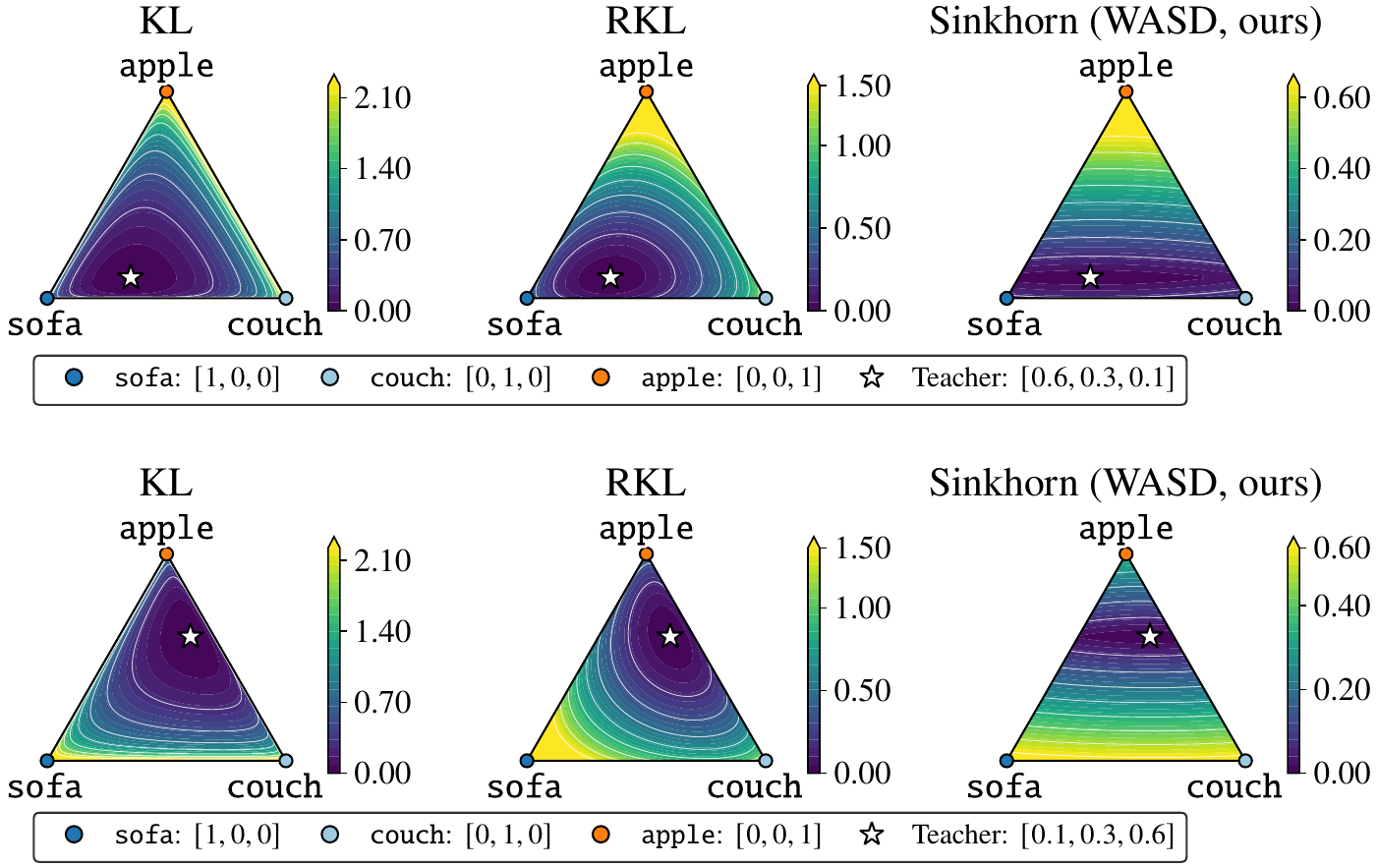}
        \caption{Loss contours of various divergences on the probability simplex}
        \label{subfig:motivation_loss_surfaces}
    \end{subfigure}
    \hfill 
    \begin{subfigure}[b]{0.27\linewidth}
        \centering
        \includegraphics[width=\linewidth]{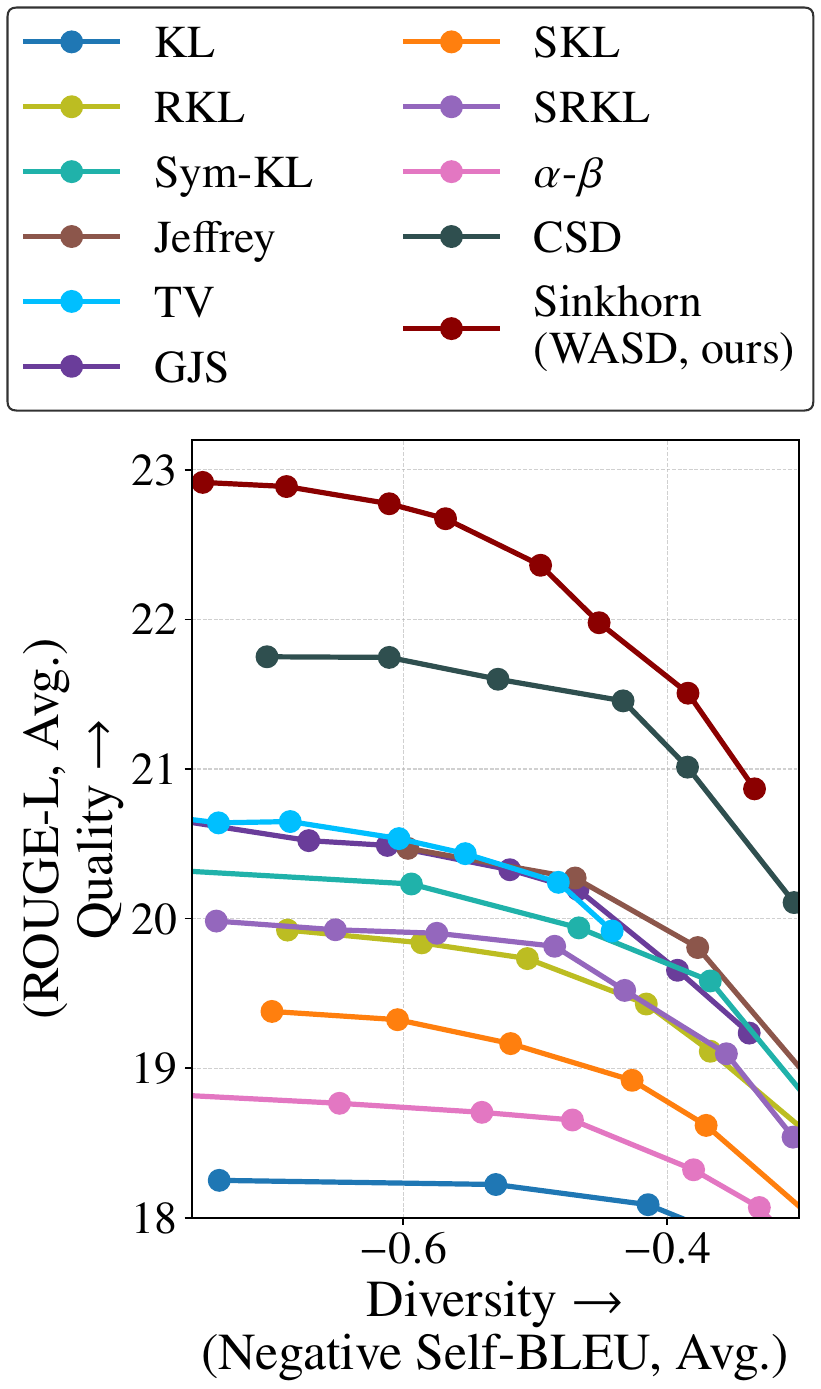}
        \caption{Quality-diversity trade-off}
        \label{subfig:motivation_trade_off}
    \end{subfigure}
    \caption{Motivation for Wasserstein-based distillation (WASD). (a) Loss contours on the probability simplex for a toy next-token prediction task (``\textit{What do you sit on in the living room?}'') with candidates \texttt{\{sofa, couch, apple\}}, where we visualize divergence between fixed teacher and varying student distributions. Unlike other divergences, Sinkhorn divergence used in WASD reflects semantic relationships between tokens. (b) Quality-diversity trade-off (ROUGE-L vs. Self-BLEU) on five instruction-following benchmarks for \texttt{GPT-2} (1.5B $\rightarrow$ 0.1B), varying the decoding temperature.}
    \label{fig:motivation}
\end{figure}

\section{Introduction}
\label{sec:intro}

Autoregressive large language models (LLMs) have demonstrated remarkable capabilities and are now widely used across a broad range of applications~\cite{hui2024qwen2,hurst2024gpt,team2024gemma,touvron2023llama2}. This success has been largely driven by scaling model size and training data, but it also incurs substantial inference costs in terms of computation and memory~\cite{hoffmann2022training,kaplan2020scaling}. Such costs pose significant challenges for practical deployment, particularly in latency-sensitive or resource-constrained settings~\cite{alizadeh2024llm,zitkovich2023rt}. Therefore, a central problem is how to preserve the performance of large models while reducing their size.

Knowledge distillation (KD) offers a promising solution by transferring knowledge from a large teacher model to a smaller student model~\cite{hinton2015distilling}. In LLMs, this process can be formulated as aligning the per-token discrete probability distributions of the teacher and the student. From this perspective, a key design choice is the divergence used to compare these distributions. Most existing approaches rely on $f$-divergences and their smoothed variants~\cite{agarwal2024onpolicy,gu2024minillm,ko2024distillm,ko2025distillm2,shing2025taid,wen2023f}, since they can be computed directly from the probability values of the two distributions. However, these divergences are computed from probability values assigned to vocabulary indices, without explicitly using the semantic information carried by tokens.
To illustrate this limitation, we consider a simple toy next-token prediction task. Given a prompt such as ``\textit{What do you sit on in the living room?}'', we examine distributions over tokens including \{\texttt{sofa, couch, apple}\} and visualize the resulting loss landscapes under different divergences. As shown in \cref{subfig:motivation_loss_surfaces}, commonly used divergences are determined by the probability assignments to vocabulary indices and do not distinguish tokens with similar meanings, such as \texttt{sofa} and \texttt{couch}. As a result, changing the teacher distribution alters the loss surface according to probability values, rather than the meaning of the corresponding tokens.

To address this limitation, we propose Wasserstein-based knowledge distillation (WASD), which incorporates token-level semantic information through a transport-based discrepancy. Specifically, we define a transport cost from token embeddings and adopt a Wasserstein-based Sinkhorn divergence to align teacher and student distributions. As illustrated in \cref{subfig:motivation_loss_surfaces}, the resulting loss landscape assigns lower cost to probability shifts between tokens with similar meaning while penalizing shifts toward unrelated tokens. This semantic-aware alignment encourages the student model to preserve meaningful alternatives during training. Empirically, this leads to improved trade-offs between output quality and diversity compared to other divergences, as shown in \cref{subfig:motivation_trade_off}.

Wasserstein-based objectives provide desirable properties but require solving an optimal transport problem. Moreover, entropy regularization makes transport computation tractable but can shift the optimum away from the teacher distribution, which is undesirable for knowledge distillation. To address these issues, we adopt the debiased Sinkhorn divergence and derive a computationally tractable optimization procedure that uses dual potentials computed via matrix scaling iterations, without introducing additional neural networks or requiring backpropagation through the transport solver.
Extensive experiments across multiple LLM families, including GPT-2~\cite{radford2019language}, OpenLLaMA2~\cite{touvron2023llama2}, Gemma~\cite{team2024gemma}, and Qwen2.5~\cite{hui2024qwen2}, and across different model scales demonstrate that WASD consistently improves distillation performance on diverse tasks, such as instruction following, mathematical reasoning, and code generation. These results suggest the semantic information encoded in the token space is an important ingredient for effective knowledge transfer in LLM distillation.

\section{Preliminaries}
\label{sec:prelim}

\begin{table}[tp]
    \centering
    \caption{Divergences used in LLM distillation. ``\textit{Token info}'' indicates whether the objective uses token information beyond probability values. For Sinkhorn, the dual potentials $\phi_i^{*}$ encode semantic information through token-embedding costs, and $\operatorname{sg}$ denotes the stop-gradient operator.}
    \adjustbox{max width=\linewidth}{%
    \begin{tabular}{lcl}
        \toprule
        Divergence & Token info & Distillation objective   \\
        \midrule
        KL~(\citeauthor{hinton2015distilling}, \citeyear{hinton2015distilling}) & \xmark & $ \sum_{i} p(i) \log \frac{p(i)}{q_{\bm{\theta}} (i)} $ \\
        Reverse KL (RKL)~(\citeauthor{gu2024minillm}, \citeyear{gu2024minillm}) & \xmark & $\sum_{i} q_{\bm{\theta}}(i) \log \frac{q_{\bm{\theta}}(i)}{p (i)} $ \\ 
        TV~(\citeauthor{wen2023f}, \citeyear{wen2023f}) & \xmark & $ \sum_{i} | q_{\bm{\theta}}(i)  - p(i) | $ \\
        Generalized JS (GJS)~(\citeauthor{agarwal2024onpolicy}, \citeyear{agarwal2024onpolicy}) & \xmark & $\sum_i [ \lambda p(i)\log\frac{p(i)}{\lambda p(i)+(1-\lambda)q_{\bm{\theta}}(i)}+(1-\lambda)q_{\bm{\theta}}(i)\log\frac{q_{\bm{\theta}}(i)}{\lambda p(i)+(1-\lambda)q_{\bm{\theta}}(i)}]$ \\ 
        Skew KL (SKL)~(\citeauthor{ko2024distillm}, \citeyear{ko2024distillm}) & \xmark & $ \sum_{i} p(i) \log \frac{p(i)}{\lambda p(i) + (1-\lambda) q_{\bm{\theta}} (i)} $ \\
        Skew RKL (SRKL)~(\citeauthor{ko2024distillm}, \citeyear{ko2024distillm}) & \xmark & $\sum_{i} (\lambda p(i) + (1-\lambda) q_{\bm{\theta}} (i)) \log \frac{\lambda p(i) + (1-\lambda) q_{\bm{\theta}} (i)}{p(i)} $ \\
        $\alpha$-$\beta$~(\citeauthor{wang2025abkd}, \citeyear{wang2025abkd}) & \xmark & $\sum_i -\frac{1}{\alpha\beta} [ p(i)^{\alpha} q_{\bm{\theta}}(i)^{\beta} - \frac{\alpha}{\alpha+\beta} p(i)^{\alpha+\beta} - \frac{\beta}{\alpha+\beta} q_{\bm{\theta}}(i)^{\alpha+\beta}  ]$ \\
        \textbf{Sinkhorn~(ours)} & \checkmark & $\sum_{i} \operatorname{sg}(\phi_i^{*(p,q_{\bm{\theta}})} - \phi_i^{*(q_{\bm{\theta}},q_{\bm{\theta}})}) q_{\bm{\theta}}(i)$  \\
        \bottomrule
    \end{tabular}
    }
    \label{tab:lst_divergence}
\end{table}

\subsection{Knowledge distillation for large language models}
\label{subsec:prelim_kd_llm}

We consider knowledge distillation (KD) for autoregressive large language models (LLMs), where the goal is to transfer knowledge from a large teacher model $p(\rvy|\rvx)$ to a smaller student model $q_{\bm{\theta}}(\rvy|\rvx)$. Here, $\rvx$ denotes the input prompt, $\rvy$ is the output response, and $\bm{\theta}$ is the parameters of the student model. The output $\rvy = (y_1, \dots, y_L)$ is a token sequence with length $L$, where each token $y_l \in \mathcal{V}$ belongs to a fixed vocabulary $\mathcal{V}$. Following prior work~\cite{kim2026distillation,ko2024distillm,shin2026amid}, we assume that the teacher and student share the same vocabulary. Autoregressive LLMs factorize the conditional distribution as $q_{\bm{\theta}}(\rvy|\rvx) = \prod_{l=1}^L q_{\bm{\theta}}(y_l|\rvx,\rvy_{<l})$, where $\rvy_{<l} = (y_1, ...,y_{l-1})$. Thus, the model predicts a next-token distribution at each decoding step.
Under this formulation, KD can be viewed as aligning the student's next-token distribution with that of the teacher. This leads to the following general objective:
\begin{align}
    \min_{{\bm{\theta}}} \mathbb{E}_{(\rvx,\rvy) \sim \mathcal{D}} \left [ \textstyle\sum_{l=1}^{L} D \Big ( p (y_{l} | \rvx, \rvy_{<l}),  q_{\bm{\theta}} (y_{l} | \rvx, \rvy_{<l})  \Big ) \right ], \label{eq:KD_general}
\end{align}
where $D(\cdot, \cdot)$ denotes a divergence between distributions. The dataset $\mathcal{D}$ consists of input-output pairs and can be constructed from predefined datasets~\cite{hinton2015distilling} or model-generated samples~\cite{agarwal2024onpolicy,ko2024distillm,lin2020autoregressive}.

The choice of divergence $D$ is a central design component in KD for LLMs, as it determines how the student aligns with the teacher distribution. \cref{tab:lst_divergence} summarizes commonly used divergences in prior work. Most existing methods adopt divergences computed from probability values at corresponding vocabulary indices, including KL divergence~\cite{hinton2015distilling,kim2016sequence}, reverse KL (RKL) divergence~\cite{gu2024minillm}, $f$-divergence~\cite{wen2023f}, generalized Jensen-Shannon (GJS) divergence~\cite{agarwal2024onpolicy}, adaptive KL divergence~\cite{wu2025rethinking}, and $\alpha$-$\beta$-divergence~\cite{wang2025abkd}. These divergences are widely used because they are tractable to compute from the probability values of the two distributions.

Despite these advantages, such divergences do not explicitly leverage token-level semantic information~\cite{na2026semanticaware}. In addition, token distributions in LLMs are highly sparse, with many near-zero probabilities, which can lead to numerical instability for divergences that depend on density ratios~\cite{ko2024distillm,shin2026amid}. To mitigate such instability, recent works have proposed alternative strategies, such as introducing an assistant distribution~\cite{ko2024distillm,ko2025distillm2,shin2026amid,shing2025taid} or aligning logits instead of probabilities~\cite{kim2026distillation}. While these approaches improve optimization stability, they still rely on index-wise comparisons and do not explicitly use semantic information encoded in the token space.

\subsection{Wasserstein distance and Sinkhorn divergence}

These limitations motivate alternative discrepancy measures for LLM distillation. The Wasserstein distance compares probability distributions by measuring the cost of transporting probability mass across the underlying space. Formally, for discrete distributions $p$ and $q_{\bm{\theta}}$, it is defined as
\begin{align}
    D_{\text{W}} \left ( p ,  q_{\bm{\theta}}  \right ) := \textstyle \min_{\mP \in U( p, q_{\bm{\theta}} )} \langle \mP, \mC \rangle,  \label{eq:wass}
\end{align}
where $U( p, q_{\bm{\theta}} )$ denotes the set of valid transport plans whose marginals match $p$ and $q_{\bm{\theta}}$, and $\mC$ is a cost matrix specifying pairwise transport costs. This formulation compares distributions through a transport plan whose cost can reflect information encoded in the support.



In practice, directly computing the Wasserstein distance is computationally expensive in large-scale settings. A common approach is to introduce entropy regularization, which enables efficient computation through iterative matrix scaling methods~\cite{cuturi2013sinkhorn,cuturi2014fast,peyre2019computational,villani2008optimal}. However, entropy regularization introduces a bias such that the transport cost between identical distributions is no longer zero. To address this issue, we adopt the Sinkhorn divergence, which removes this bias through self-similarity correction while retaining the computational advantages of entropy-regularized optimal transport~\cite{genevay2018learning}. 
Sinkhorn divergences have been studied in generative modeling and distribution alignment due to their scalability and favorable optimal transport properties~\cite{feydy2019interpolating,genevay2018learning,sanjabi2018convergence}. In the context of LLM distillation, such transport-based discrepancies provide a principled way to incorporate relationships within the vocabulary. Building on this perspective, we introduce a Sinkhorn-based distillation objective for autoregressive LLMs, whose detailed formulations and derivations are presented in \cref{subsec:method_objective}.

\paragraph{Wasserstein distance in knowledge distillation}
Wasserstein-based distances have been used in KD to transfer structured knowledge. Earlier works mainly target representations or features~\cite{bhardwaj2022knot,chen2021wasserstein,lohit2022model}. Recent methods use Wasserstein-based distillation losses: SinKD~\cite{cui2024sinkhorn} aligns sample-level geometry, and WKD~\cite{lv2024wasserstein} exploits class interrelations tailored for vision models. They have also been used for cross-tokenizer LLM distillation~\cite{boizard2025towards,cui2025multi,vuong2026mcw}, which addresses vocabulary mismatch or heterogeneous alignment and is complementary to our shared-vocabulary setting.
Outside KD, WPR~\cite{na2026semanticaware} uses token-embedding geometry for Wasserstein-based policy regularization in preference learning. WASD builds on this geometric perspective for autoregressive teacher-student distribution matching and adopts a debiased Sinkhorn divergence to address the entropic bias of regularized transport. A detailed comparison with WPR and discussion of other related work are provided in \cref{app_sec:related_work}.

 
\section{Method}
\label{sec:method}


\subsection{From Wasserstein distance to Sinkhorn divergence for LLM distillation}
\label{subsec:method_objective}

\paragraph{Incorporating token-level semantic information via Wasserstein distance}

We instantiate the general distillation objective in \cref{eq:KD_general} using the Wasserstein distance $D_{\text{W}}$ in \cref{eq:wass} to incorporate token-space information into the distillation objective. In particular, the cost matrix $\mC$, which encodes pairwise distances between tokens, allows the transport cost to depend on information encoded in token embeddings. In practice, we construct $\mC$ using distances between the teacher model's token embeddings (e.g., cosine or Euclidean distance), in line with the goal of transferring knowledge from the teacher to the student. We compute this cost matrix once and keep it fixed throughout training.

\paragraph{Entropy-regularized Wasserstein distance}
A key challenge is that computing the exact Wasserstein distance requires solving a large-scale linear optimization problem, which becomes computationally intractable as the support size increases~\cite{kuhn2019wasserstein}. To alleviate this issue, a common approach is to regularize the transport problem with an entropy term, resulting in the entropy-regularized Wasserstein distance~\cite{cuturi2013sinkhorn}:
\begin{gather}
    \mathcal{L}_{\tilde{\text{W}}}^{\epsilon}(\bm{\theta}) := \mathbb{E}_{(\rvx,\rvy) \sim \mathcal{D}} \left [ \textstyle\sum_{l=1}^{L} \tilde{D}_{{\text{W}}}^{\epsilon} \Big ( p (y_{l} | \rvx, \rvy_{<l}),  q_{\bm{\theta}} (y_{l} | \rvx, \rvy_{<l})  \Big ) \right ], \label{eq:KD_entW} \\
    \text{where }
    {\tilde{D}}_{{\text{W}}}^{\epsilon}\Big ( p (y_{l} | \rvx, \rvy_{<l}),  q_{\bm{\theta}} (y_{l} | \rvx, \rvy_{<l})  \Big ) :=  \textstyle \min_{\mP^{(l)} \in U_l( p, q_{\bm{\theta}} )} \left \{ \langle \mP^{(l)}, \mC \rangle - \epsilon \mathcal{H}(\mP^{(l)}) \right \}. \label{eq:entropic_wass}
\end{gather}
Here, $U_l( p, q_{\bm{\theta}} ) := \left \{ \mP^{(l)} \in \mathbb{R}_{+}^{d\times d} | \mP^{(l)} \bm{1}_d = q_{\bm{\theta}} ( \cdot | \rvx, \rvy_{<l}), {\mP^{(l)}}^{\top} \bm{1}_d = p ( \cdot | \rvx, \rvy_{<l})  \right \}$ denotes the set of valid transport plans, $\mC \in \mathbb{R}_{+}^{d\times d}$ is the cost matrix, $d$ is the vocabulary size, and $\epsilon$ is an entropy regularization hyperparameter. The entropy term is defined as $\mathcal{H}(\mP) := -\sum_{i=1}^{d}\sum_{j=1}^{d} \emP_{ij} (\log \emP_{ij} - 1)$, which corresponds to the Shannon entropy up to an additive constant. For notational simplicity, we omit the dependence of $\mP^{(l)}$ on $(\rvx,\rvy_{<l})$.

\begin{wrapfigure}[10]{R}{0.27\textwidth}
    \vspace{-1.5em}
    \centering
    \includegraphics[width=\linewidth]{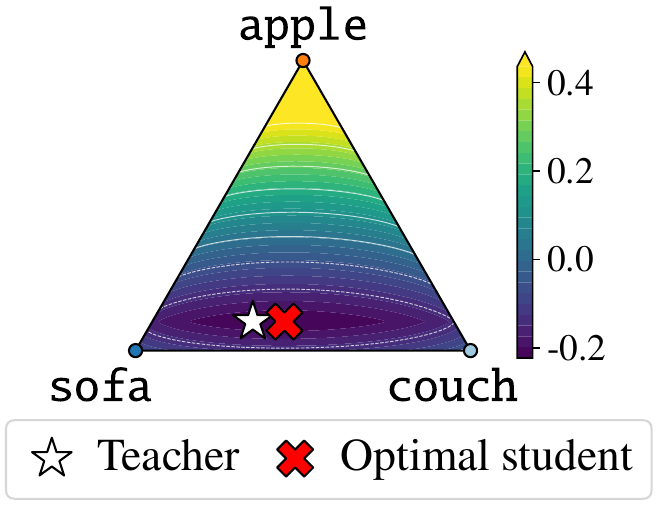}
    \caption{Loss contour of entropy-regularized Wasserstein distance.}
    \label{fig:ent_wass}
\end{wrapfigure}
\paragraph{Bias induced by entropy regularization}
By introducing entropy regularization, the resulting objective can be computed efficiently using the Sinkhorn-Knopp algorithm~\cite{sinkhorn1967concerning}.
However, this regularization introduces an entropic bias: for a fixed teacher distribution $p$, the regularized transport objective is not generally minimized at $q_{\bm{\theta}}=p$.
Consequently, minimizing the objective in \cref{eq:KD_entW} does not guarantee that the student distribution $q_{\bm{\theta}}$ matches the teacher distribution $p$, which is the central goal of knowledge distillation. As illustrated in \cref{fig:ent_wass} on the toy example introduced in \cref{sec:intro}, the optimal solution of the entropy-regularized Wasserstein objective deviates from the teacher distribution, highlighting the bias introduced by entropy regularization.

\paragraph{Sinkhorn divergence for proper distillation}
To address this limitation, we adopt the Sinkhorn divergence~\cite{feydy2019interpolating,genevay2018learning}, defined as
\begin{gather}
    \mathcal{L}_{\text{S}}^{\epsilon}(\bm{\theta}) := \mathbb{E}_{(\rvx,\rvy) \sim \mathcal{D}} \left [  \textstyle\sum_{l=1}^{L} {D}_{{\text{S}}}^{\epsilon} \Big ( p (y_{l} | \rvx, \rvy_{<l}),  q_{\bm{\theta}} (y_{l} | \rvx, \rvy_{<l})  \Big ) \right ], \label{eq:KD_sink} \\
    \text{where }
    {D}_{{\text{S}}}^{\epsilon}( p ,  q_{\bm{\theta}} ) := {\tilde{D}}_{{\text{W}}}^{\epsilon}( p,  q_{\bm{\theta}})
    - \frac{1}{2} {\tilde{D}}_{{\text{W}}}^{\epsilon}( q_{\bm{\theta}},  q_{\bm{\theta}})
    - \frac{1}{2} {\tilde{D}}_{{\text{W}}}^{\epsilon}( p, p).\label{eq:sinkhorn}
\end{gather}
Under suitable conditions on the transport cost, the Sinkhorn divergence in \cref{eq:sinkhorn} is nonnegative and equals zero if and only if $p=q_{\bm{\theta}}$. Then, its unique minimizer over student distributions is the teacher distribution, making \cref{eq:KD_sink} a proper distillation objective. The entropy regularization parameter $\epsilon$ interpolates between different regimes: as $\epsilon \to 0$, $D_{\text{S}}^{\epsilon}$ recovers the Wasserstein distance $D_{\text{W}}$; while as $\epsilon \to \infty$, it approaches a kernel-based discrepancy, recovering the squared Maximum Mean Discrepancy (MMD) for suitable choices of the cost matrix~\cite{feydy2019interpolating}.

\subsection{WASD: Wasserstein-based knowledge distillation for LLMs}
\label{subsec:wasd}

We derive a tractable optimization procedure for the Sinkhorn distillation objective in \cref{eq:KD_sink}, which can be directly optimized with standard gradient-based methods, without introducing additional model components. Since the Sinkhorn divergence ${D}_{{\text{S}}}^{\epsilon}$ is defined in terms of entropy-regularized Wasserstein distances, we begin by reformulating the entropy-regularized Wasserstein distance ${\tilde{D}}_{{\text{W}}}^{\epsilon}$ to obtain a tractable distillation formulation. 

We consider the dual formulation of the entropy-regularized optimal transport problem in \cref{eq:entropic_wass}.
The corresponding dual problem can be written as follows~\cite{cuturi2013sinkhorn,na2026semanticaware}:
\begin{align}
    \max_{\bm{\phi},\bm{\psi}} \sum_{i=1}^{d} \phi_i q_{\bm{\theta}}(y_l=i|\rvx,\rvy_{<l}) + \sum_{j=1}^{d} \psi_j p(y_l=j|\rvx,\rvy_{<l}) - \sum_{i=1}^{d} \sum_{j=1}^{d} \epsilon \exp\!\big(\frac{1}{\epsilon} (\phi_i + \psi_j - C_{ij} ) \big),\label{eq:eot_dual}
\end{align}
where $\bm{\phi} = \{\phi_i\}_{i=1}^d$ and $\bm{\psi} = \{\psi_j\}_{j=1}^d$ are dual variables associated with the marginal constraints in $U_l$, ensuring that the row and column sums of $\mP^{(l)}$ match $q_{\bm{\theta}}$ and $p$, respectively. These variables depend on $(\rvx,\rvy_{<l})$, which we omit for notational simplicity.
By strong duality, the optimal primal and dual solutions admit the following form~\citep{cuturi2014fast}: there exist scaling vectors $\mathbf{u}, \mathbf{v} \in \mathbb{R}^{d}_{+}$ such that
\begin{align}
    {\mP^{(l)}}^* = \mathrm{diag}(\mathbf{u}) \exp (-\mC/\epsilon) \mathrm{diag}(\mathbf{v}), \quad \bm{\phi}^* = \epsilon \log (\mathbf{u})+t\mathbf{1}_d, \quad \bm{\psi}^* = {\epsilon} \log (\mathbf{v})-t\mathbf{1}_d,
    \label{eq:optimal_sol}
\end{align}
for $t \in \mathbb{R}$. Here, $\exp$ and $\log$ are applied element-wise. The dual solutions are unique up to an additive constant. This indeterminacy does not affect the dual objective, since both marginals have unit mass and the shifts cancel.

Substituting the optimal solution into the dual objective in \cref{eq:eot_dual} yields the following equivalent expression for the entropy-regularized Wasserstein distance.
\begin{restatable}{lemma}{lema}
\label{thm:entW_objective}   
    Let $\bm{\phi}^{*(p,q_{\bm{\theta}})}(\rvx,\rvy_{<l})$ and $\bm{\psi}^{*(p,q_{\bm{\theta}})}(\rvx,\rvy_{<l})$ denote the optimal dual variables of the entropic optimal transport problem corresponding to ${\tilde{D}}_{{\text{W}}}^{\epsilon}( p,  q_{\bm{\theta}})$. Then,
    \begin{align}
        {\tilde{D}}_{{\text{W}}}^{\epsilon}(p&(y_l|\rvx,\rvy_{<l}), q_{\bm{\theta}}(y_l|\rvx,\rvy_{<l}))  \label{eq:KD_entWd} \\ 
        &= \textstyle\sum_{i=1}^{d} \phi_i^{*(p,q_{\bm{\theta}})}(\rvx,\rvy_{<l}) q_{\bm{\theta}}(y_l=i|\rvx,\rvy_{<l})+
        \sum_{j=1}^{d} \psi_j^{*(p,q_{\bm{\theta}})}(\rvx,\rvy_{<l}) p(y_l=j|\rvx,\rvy_{<l}) - \epsilon. \nonumber
    \end{align}
\end{restatable}
The proof is provided in \cref{app_subsec:proof_entW}. 
The optimal dual variables depend implicitly on $\bm{\theta}$ through $q_{\bm{\theta}}$. However, the envelope theorem~\citep{milgrom2002envelope} allows us to compute the gradient of the optimal transport value without differentiating through these variables, as formalized below.
\begin{restatable}{proposition}{gradient}
\label{thm:gradient}
    Let $\bm{\phi}^{*(p,q_{\bm{\theta}})}(\rvx,\rvy_{<l})$ denote the optimal dual variables associated with the student marginal. Then, the gradient of the entropy-regularized Wasserstein distance can be expressed as
    \begin{align}
        & \nabla_{\bm{\theta}} \tilde{D}_{\text{W}}^{\epsilon} (p(y_l|\rvx,\rvy_{<l}), q_{\bm{\theta}}(y_l|\rvx,\rvy_{<l})) = \textstyle\sum_{i=1}^{d}  \phi_i^{*(p,q_{\bm{\theta}})}(\rvx,\rvy_{<l})\nabla_{\bm{\theta}}  q_{\bm{\theta}}(y_l=i|\rvx,\rvy_{<l}) . \label{eq:entW_gradient} 
    \end{align}
\end{restatable}
The proof is provided in \cref{app_subsec:proof_gradient}. \cref{thm:gradient} shows that the gradient of the entropy-regularized Wasserstein distance can be expressed as a weighted sum of student probability gradients, with weights given by the optimal dual variables. This allows gradient-based optimization without differentiating through the transport solver.
We now extend this formulation to the Sinkhorn distillation objective in \cref{eq:KD_sink}. The self-transport terms in \cref{eq:sinkhorn} can be treated analogously using \cref{thm:entW_objective,thm:gradient}, leading to the following result.
\begin{restatable}{theorem}{thma}
\label{thm:sink_objective}   
    Let $\bm{\phi}^{*(p,q_{\bm{\theta}})}(\rvx,\rvy_{<l})$ and $\bm{\phi}^{*(q_{\bm{\theta}},q_{\bm{\theta}})}(\rvx,\rvy_{<l})$ denote the optimal dual variables corresponding to ${\tilde{D}}_{{\text{W}}}^{\epsilon}( p,  q_{\bm{\theta}})$ and ${\tilde{D}}_{{\text{W}}}^{\epsilon}( q_{\bm{\theta}},  q_{\bm{\theta}})$, respectively. Define
    \begin{align}
        \mathcal{L}_{\text{WASD}}^{\epsilon}(\bm{\theta}) := \mathbb{E}_{(\rvx,\rvy) \sim \mathcal{D}} \Bigg [ \sum_{l=1}^{L} \sum_{i=1}^{d} \operatorname{sg} \Big  (\phi_i^{*(p,q_{\bm{\theta}})}(\rvx,\rvy_{<l}) -  \phi_i^{*(q_{\bm{\theta}},q_{\bm{\theta}})}(\rvx,\rvy_{<l}) \Big ) q_{\bm{\theta}}(y_l=i|\rvx,\rvy_{<l}) \Bigg] ,  \label{eq:KD_wasd}
    \end{align}
    where $\operatorname{sg}$ denotes the stop-gradient operator. Then, under the stop-gradient convention, the gradient of $\mathcal{L}_{\textrm{WASD}}^{\epsilon}(\bm{\theta})$ coincides with that of the Sinkhorn objective $\mathcal{L}_{{\text{S}}}^{\epsilon}(\bm{\theta})$, i.e., $\nabla_{\bm{\theta}} \mathcal{L}_{\textrm{WASD}}^{\epsilon}(\bm{\theta})= \nabla_{\bm{\theta}} \mathcal{L}_{{\text{S}}}^{\epsilon}(\bm{\theta})$.
\end{restatable}
The result follows by differentiating the three terms in \cref{eq:sinkhorn}. By \cref{thm:gradient}, the first term yields a student probability gradient weighted by $\bm{\phi}^{*(p,q_{\bm{\theta}})}$, and the third term is independent of $\bm{\theta}$. The second term involves the entropy-regularized Wasserstein distance between $q_{\bm{\theta}}$ and itself. Assuming a symmetric cost matrix, the optimal dual variables in this self-transport problem can be chosen equal, i.e., $\bm{\phi}^{*(q_{\bm{\theta}},q_{\bm{\theta}})} = \bm{\psi}^{*(q_{\bm{\theta}},q_{\bm{\theta}})}$, allowing its gradient to be expressed using only $\bm{\phi}^{*(q_{\bm{\theta}},q_{\bm{\theta}})}$. Combining these contributions gives a gradient weighted by the difference between cross-transport and student self-transport potentials. Applying stop-gradient to these weights gives the loss in \cref{eq:KD_wasd}, whose gradient coincides with that of the Sinkhorn objective. The full proof is in \cref{app_subsec:proof_sink}. 

We refer to the objective $\mathcal{L}_{\text{WASD}}^{\epsilon}(\bm{\theta})$ in \cref{eq:KD_wasd} as the \textbf{WAS}serstein-based \textbf{D}istillation (WASD) objective. By \cref{thm:sink_objective}, WASD can be interpreted as a weighted sum over $q_{\bm{\theta}}$, where weights are given by the difference between the optimal dual variables $\bm{\phi}^{*(p,q_{\bm{\theta}})}(\rvx,\rvy_{<l})$ and $\bm{\phi}^{*(q_{\bm{\theta}},q_{\bm{\theta}})}(\rvx,\rvy_{<l})$, treated as constants during backpropagation. These dual variables can be efficiently computed via the Sinkhorn-Knopp algorithm (see \cref{alg:sinkhorn_knopp}) and a fixed-point iteration (see \cref{alg:fixed_point}), respectively. To improve computational efficiency, we follow \citet{na2026semanticaware} to construct a sparse kernel by retaining each token's $k$ nearest non-self neighbors in the embedding space, symmetrizing the support, and adding self-loops. Although the matrix scaling iterations in \cref{alg:sinkhorn_knopp,alg:fixed_point} introduce additional computational overhead, WASD avoids differentiation through them, as established in \cref{thm:sink_objective}.

\begin{figure}[t]
    \begin{minipage}{0.53\textwidth}
        \begin{algorithm}[H]
            \caption{Sinkhorn-Knopp algorithm for $\bm{\phi}^{*(p,q_{\bm{\theta}})}$}
            \label{alg:sinkhorn_knopp}
            \begin{algorithmic}[1]
                \REQUIRE Teacher distribution $p(\cdot|\rvx,\rvy_{<l})$, student distribution $q_{\bm{\theta}}(\cdot|\rvx,\rvy_{<l})$, cost matrix $\mC$, regularization hyperparameter $\epsilon$
                \STATE $\mathbf{u} \leftarrow \mathbf{1}_{d}$, $\mathbf{v} \leftarrow \mathbf{1}_{d}$, $\mK \leftarrow \exp(-\mC / \epsilon)$
                \WHILE{not converged}
                    \STATE $\mathbf{u} \leftarrow q_{\bm{\theta}}(\cdot|\rvx,\rvy_{<l}) ~./~ (\mK \mathbf{v})$
                    \STATE $\mathbf{v} \leftarrow p(\cdot|\rvx,\rvy_{<l}) ~./~ (\mK^{\top} \mathbf{u})$
                \ENDWHILE
                \STATE $\bm{\phi}^{*(p,q_{\bm{\theta}})} \leftarrow\epsilon \log \mathbf{u}$
                \ENSURE Dual variable $\bm{\phi}^{*(p,q_{\bm{\theta}})}$
            \end{algorithmic}
        \end{algorithm}
    \end{minipage}
    \hfill
    \begin{minipage}{0.45\textwidth}
        \begin{algorithm}[H]
            \caption{Fixed-point iteration for $\bm{\phi}^{*(q_{\bm{\theta}},q_{\bm{\theta}})}$}
            \label{alg:fixed_point}
            \begin{algorithmic}[1]
                \REQUIRE Student distribution $q_{\bm{\theta}}(\cdot|\rvx,\rvy_{<l})$, cost matrix $\mC$, regularization hyperparameter $\epsilon$
                \STATE $\mathbf{u} \leftarrow \mathbf{1}_{d}$, $\mK \leftarrow \exp(-\mC / \epsilon)$
                \WHILE{not converged}
                    \STATE $\mathbf{u} \leftarrow q_{\bm{\theta}}(\cdot|\rvx,\rvy_{<l}) ~./~ (\mK \mathbf{u})$
                \ENDWHILE
                \STATE $\bm{\phi}^{*(q_{\bm{\theta}},q_{\bm{\theta}})} \leftarrow\epsilon \log \mathbf{u}$
                \ENSURE Dual variable $\bm{\phi}^{*(q_{\bm{\theta}},q_{\bm{\theta}})}$
            \end{algorithmic}
        \end{algorithm}
    \end{minipage}
\end{figure}

\paragraph{Gradient update of WASD}

We compare the gradient update of WASD with that of KL-based distillation. For a fixed context $(\rvx,\rvy_{<l})$, using $\nabla_{\bm{\theta}} q_{\bm{\theta}} = q_{\bm{\theta}} \nabla_{\bm{\theta}}\log q_{\bm{\theta}}$, the Sinkhorn gradient implemented by WASD in \cref{thm:sink_objective} can be written as
\begin{align}
    \nabla_{\bm{\theta}} D_{\text{S}}^{\epsilon} (&  p(y_l | \rvx, \rvy_{<l}), q_{\bm{\theta}} (y_l | \rvx,\rvy_{<l}) ) \label{eq:wasd_log_gradient} \\
    &= \sum_{i=1}^d q_{\bm{\theta}} (y_l = i | \rvx, \rvy_{<l}) \left (  \phi_i^{*(p,q_{\bm{\theta}})}(\rvx,\rvy_{<l}) - \phi_i^{*(q_{\bm{\theta}},q_{\bm{\theta}})}(\rvx,\rvy_{<l})  \right ) \nabla_{\bm{\theta}} \log q_{\bm{\theta}}(y_l=i|\rvx,\rvy_{<l}).  \nonumber
\end{align}
The gradient is an expectation under the student distribution of log-likelihood gradients weighted by the difference of dual potentials. For comparison, the gradient of the forward KL objective is
\begin{align}
    \nabla_{\bm{\theta}} D_{\text{KL}}(p(y_l|\rvx,\rvy_{<l}&), q_{\bm{\theta}} (y_l|\rvx,\rvy_{<l}))  
    = -\textstyle\sum_{i=1}^{d}  p(y_l=i|\rvx,\rvy_{<l})   \nabla_{\bm{\theta}} \log q_{\bm{\theta}}(y_l=i|\rvx,\rvy_{<l}),
\end{align}
which weights the negative log-likelihood gradient directly by the teacher probability at each vocabulary index. In contrast, WASD weights the log-likelihood gradient by the student probability multiplied by the difference $\phi_i^{*(p,q_{\bm{\theta}})} - \phi_i^{*(q_{\bm{\theta}},q_{\bm{\theta}})} $. These potentials depend on the token embedding cost matrix, allowing token geometry to influence the update. Consequently, while KL determines its weights from index-wise probability values, WASD additionally incorporates semantic relationships between tokens through the transport problem.

\begin{table}[tp]
    \centering
    \caption{ROUGE-L scores on five general instruction-following benchmarks for \texttt{GPT-2 XL (1.5B) $\rightarrow$ GPT-2 Base (0.1B)}, varying divergences.  All results are obtained using our implementation, where all methods share identical teacher and student initialization. Performance is averaged over five evaluation seeds. \textbf{Bold} and \underline{underline} indicate the best and second-best performance, respectively.}
    \vspace{-.5em}
    \resizebox{\textwidth}{!}{%
    \begin{tabular}{lcccccc}
        \toprule
        Method & Dolly Eval & Self Inst & Vicuna & Super NI & UnNI & Avg. ($\uparrow$) \\
        \midrule
        GPT-2 XL (Teacher) & \rateinline{26.93}{0.60} & \rateinline{14.77}{0.38} & \rateinline{16.43}{0.28} & \rateinline{27.19}{0.47} & \rateinline{31.43}{0.12} & 23.35 \\
        \midrule
        
        KL~\cite{hinton2015distilling} & \rateinline{23.33}{0.36} & \rateinline{10.43}{0.23} & \rateinline{15.34}{0.27} & \rateinline{17.42}{0.14} & \rateinline{19.34}{0.07} & 17.17 \\
        RKL~\cite{gu2024minillm} & \rateinline{24.62}{0.27} & \rateinline{11.31}{0.28} & \rateinline{15.97}{0.28} & \rateinline{21.42}{0.30} & \rateinline{23.82}{0.14} & 19.43 \\
        Sym-KL & \rateinline{23.32}{0.37} & \rateinline{10.81}{0.27} & \rateinline{15.41}{0.56} & \rateinline{20.12}{0.23} & \rateinline{21.23}{0.09} & 18.18 \\
        Jeffrey & \rateinline{22.84}{0.53} & \rateinline{10.84}{0.24} & \rateinline{15.30}{0.22} & \rateinline{20.57}{0.37} & \rateinline{22.66}{0.13} & 18.44 \\
        TV~\cite{wen2023f} & \rateinline{24.08}{0.41} & \rateinline{10.83}{0.25} & \rateinline{14.68}{0.43} & \rateinline{25.57}{0.22} & \rateinline{\underline{27.51}}{0.09} & 20.53 \\
        GJS~\cite{agarwal2024onpolicy} & \rateinline{24.64}{0.40} & \rateinline{11.64}{0.37} & \rateinline{14.96}{0.42} & \rateinline{24.31}{0.09} & \rateinline{26.09}{0.12} & 20.33 \\
        SKL~\cite{ko2024distillm} & \rateinline{24.51}{0.37} & \rateinline{12.00}{0.18} & \rateinline{15.00}{0.23} & \rateinline{20.26}{0.19} & \rateinline{22.83}{0.04} & 18.92 \\
        SRKL~\cite{ko2024distillm} & \rateinline{24.68}{0.19} & \rateinline{12.60}{0.24} & \rateinline{{15.15}}{0.54} & \rateinline{22.39}{0.19} & \rateinline{24.26}{0.11} & 19.81 \\
        $\alpha$-$\beta$~\cite{wang2025abkd} & \rateinline{24.05}{0.18} & \rateinline{10.88}{0.31} & \rateinline{{15.87}}{0.23} & \rateinline{19.31}{0.21} & \rateinline{21.50}{0.11} & 18.32 \\
        CSD~\cite{kim2026distillation} & \rateinline{\underline{25.06}}{0.17} & \rateinline{\underline{12.75}}{0.26} & \rateinline{\textbf{16.42}}{0.39} & \rateinline{\underline{25.86}}{0.31} & \rateinline{27.19}{0.13} & \underline{21.46} \\
        \textbf{WASD (ours)} & \rateinline{\textbf{25.60}}{0.34} & \rateinline{\textbf{12.88}}{0.36} & \rateinline{\underline{15.98}}{0.40} & \rateinline{\textbf{28.63}}{0.17} & \rateinline{\textbf{29.27}}{0.09} & \textbf{22.47} \\
        \bottomrule
    \end{tabular}%
    }
    \label{tab:exp_divergence}
    \vspace{-.5em}
\end{table}

\section{Experiments}
\label{sec:exp}

We evaluate WASD on general instruction-following and task-specific distillation tasks, including summarization, translation, mathematical reasoning, and code generation. 
Experimental setups follow prior work~\cite{kim2026distillation,shin2026amid}, and details are provided in each subsection and \cref{app_sec:exp_setting}.

\begin{table}[tp]
    \centering
    \caption{ROUGE-L scores on five general instruction-following benchmarks for \texttt{GPT-2} and \texttt{OpenLLaMA2}. \textbf{Bold} and \underline{underline} indicate the best and second-best performance, respectively.}
    \vspace{-.5em}
    \resizebox{\textwidth}{!}{%
    \begin{tabular}{lcccccc}
        \toprule
        Method & Dolly Eval & Self Inst & Vicuna & Super NI & UnNI & Avg. ($\uparrow$) \\
        \midrule
        GPT-2 XL (Teacher) & \rateinline{26.93}{0.60} & \rateinline{14.77}{0.38} & \rateinline{16.43}{0.28} & \rateinline{27.19}{0.47} & \rateinline{31.43}{0.12} & 23.35 \\
        
        \midrule
        \multicolumn{7}{l}{\texttt{GPT-2 XL (1.5B) $\rightarrow$ GPT-2 Base (0.1B)}} \\
        \midrule
        GKD~\cite{agarwal2024onpolicy} & \rateinline{23.86}{0.22} & \rateinline{11.21}{0.42} & \rateinline{14.63}{0.19} & \rateinline{19.43}{0.17} & \rateinline{22.28}{0.12} & 18.28 \\
        TAID~\cite{shing2025taid} & \rateinline{25.44}{0.42} & \rateinline{12.55}{0.11} & \rateinline{16.65}{0.32} & \rateinline{24.45}{0.20} & \rateinline{26.70}{0.09} & 21.16 \\
        DistiLLM (SKL)~\cite{ko2024distillm} & \rateinline{25.39}{0.43} & \rateinline{12.42}{0.25} & \rateinline{16.21}{0.33} & \rateinline{25.13}{0.22} & \rateinline{26.94}{0.18} & 21.22 \\
        DistiLLM (SRKL)~\cite{ko2024distillm} & \rateinline{25.10}{0.30} & \rateinline{12.01}{0.26} & \rateinline{\textbf{17.40}}{0.28} & \rateinline{25.39}{0.21} & \rateinline{26.40}{0.10} & 21.26 \\
        ABKD~\cite{wang2025abkd} & \rateinline{26.15}{0.25} & \rateinline{12.56}{0.48} & \rateinline{17.24}{0.60} & \rateinline{27.07}{0.16} & \rateinline{29.03}{0.13} & 22.41 \\
        CSD~\cite{kim2026distillation} & \rateinline{25.91}{0.52} & \rateinline{12.97}{0.19} & \rateinline{16.80}{0.57} & \rateinline{26.90}{0.16} & \rateinline{28.51}{0.06} & 22.22 \\
        AMiD~\cite{shin2026amid} & \rateinline{\underline{26.39}}{0.40} & \rateinline{\underline{14.18}}{0.17} & \rateinline{16.81}{0.25} & \rateinline{\underline{29.11}}{0.16} & \rateinline{\underline{30.78}}{0.13} & \underline{23.46} \\
        \textbf{WASD (ours)} & \rateinline{\textbf{27.07}}{0.22} & \rateinline{\textbf{14.34}}{0.49} & \rateinline{\underline{17.25}}{0.37} & \rateinline{\textbf{29.60}}{0.22} & \rateinline{\textbf{31.86}}{0.08} & \textbf{24.02} \\
        
        \midrule
        \multicolumn{7}{l}{\texttt{GPT-2 XL (1.5B) $\rightarrow$ GPT-2 Medium (0.3B)}} \\
        \midrule
        GKD~\cite{agarwal2024onpolicy} & \rateinline{25.31}{0.43} & \rateinline{12.16}{0.57} & \rateinline{15.26}{0.34} & \rateinline{26.00}{0.20} & \rateinline{27.58}{0.09} & 21.26 \\
        TAID~\cite{shing2025taid} & \rateinline{26.97}{0.18} & \rateinline{14.63}{0.42} & \rateinline{\underline{17.58}}{0.26} & \rateinline{26.75}{0.30} & \rateinline{31.01}{0.08} & 23.39 \\
        DistiLLM (SKL)~\cite{ko2024distillm} & \rateinline{25.94}{0.31} & \rateinline{14.27}{0.43} & \rateinline{16.60}{0.47} & \rateinline{25.42}{0.24} & \rateinline{30.12}{0.11} & 22.47 \\
        DistiLLM (SRKL)~\cite{ko2024distillm} & \rateinline{26.97}{0.32} & \rateinline{13.93}{0.37} & \rateinline{17.02}{0.21} & \rateinline{26.72}{0.20} & \rateinline{29.43}{0.08} & 22.81 \\
        ABKD~\cite{wang2025abkd} & \rateinline{27.18}{0.30} & \rateinline{14.12}{0.27} & \rateinline{17.28}{0.26} & \rateinline{26.87}{0.32} & \rateinline{31.14}{0.04} & 23.32 \\
        CSD~\cite{kim2026distillation} & \rateinline{27.15}{0.46} & \rateinline{\underline{15.19}}{0.36} & \rateinline{17.56}{0.39} & \rateinline{27.92}{0.23} & \rateinline{33.12}{0.09} & 24.19 \\
        AMiD~\cite{shin2026amid} & \rateinline{\underline{27.82}}{0.25} & \rateinline{{14.84}}{0.33} & \rateinline{16.55}{0.38} & \rateinline{\textbf{30.30}}{0.17} & \rateinline{\underline{34.17}}{0.06} & \underline{24.74} \\
        \textbf{WASD (ours)} & \rateinline{\textbf{28.37}}{0.14} & \rateinline{\textbf{15.45}}{0.41} & \rateinline{\textbf{18.41}}{0.28} & \rateinline{\underline{30.24}}{0.19} & \rateinline{\textbf{35.06}}{0.03} & \textbf{25.51} \\
        
        \midrule
        \midrule
        OpenLLaMA2-7B (Teacher) &  \rateinline{27.44}{0.38} & \rateinline{18.47}{0.83} & \rateinline{17.95}{0.30} & \rateinline{31.06}{0.11} & \rateinline{32.44}{0.11} & 25.47 \\
        \midrule
        \multicolumn{7}{l}{\texttt{OpenLLaMA2-7B $\rightarrow$ OpenLLaMA2-3B}} \\
        \midrule
        GKD~\cite{agarwal2024onpolicy} & \rateinline{29.36}{0.20} & \rateinline{\underline{20.73}}{0.23} & \rateinline{\underline{20.81}}{0.18} & \rateinline{\underline{38.22}}{0.22} & \rateinline{37.25}{0.14} & 29.27 \\
        TAID~\cite{shing2025taid} & \rateinline{26.27}{0.48} & \rateinline{18.10}{0.66} & \rateinline{18.05}{0.35} & \rateinline{31.58}{0.15} & \rateinline{30.10}{0.11} & 24.82 \\
        DistiLLM (SKL)~\cite{ko2024distillm} & \rateinline{28.48}{0.44} & \rateinline{\textbf{20.76}}{0.39} & \rateinline{19.48}{0.51} & \rateinline{35.44}{0.24} & \rateinline{34.54}{0.16} & 27.74 \\
        DistiLLM (SRKL)~\cite{ko2024distillm} & \rateinline{28.55}{0.39} & \rateinline{19.71}{0.52} & \rateinline{19.94}{0.34} & \rateinline{35.38}{0.25} & \rateinline{36.02}{0.18} & 27.92 \\
        ABKD~\cite{wang2025abkd} & \rateinline{29.04}{0.26} & \rateinline{20.26}{0.32} & \rateinline{20.77}{0.33} & \rateinline{37.07}{0.41} & \rateinline{36.55}{0.15} & 28.74 \\
        CSD~\cite{kim2026distillation} & \rateinline{29.08}{0.24} & \rateinline{20.27}{0.37} & \rateinline{20.26}{0.21} & \rateinline{36.96}{0.10} & \rateinline{35.92}{0.11} & 28.50 \\
        AMiD~\cite{shin2026amid} & \rateinline{\textbf{30.21}}{0.34} & \rateinline{20.65}{0.31} & \rateinline{{20.51}}{0.26} & \rateinline{{37.49}}{0.13} & \rateinline{\underline{37.61}}{0.09} & \underline{29.30} \\
        \textbf{WASD (ours)} & \rateinline{\underline{29.83}}{0.25} & \rateinline{{20.33}}{0.54} & \rateinline{\textbf{21.34}}{0.29} & \rateinline{\textbf{38.83}}{0.17} & \rateinline{\textbf{37.74}}{0.09} & \textbf{29.61} \\
        \bottomrule
    \end{tabular}%
    }
    \label{tab:main}
    \vspace{-.5em}
\end{table}

\subsection{General instruction-following experiments}
\label{subsec:exp_general_inst}

\paragraph{Setup}
For general instruction-following experiments, we follow the experimental protocols of previous work~\cite{kim2026distillation,ko2024distillm,shin2026amid}. For GPT-2 and OpenLLaMA2, we use databricks-dolly-15k~\cite{DatabricksBlog2023DollyV2} as the distillation dataset and OpenWebText~\cite{Gokaslan2019OpenWeb} as the pretraining dataset. In the divergence comparison experiment in \cref{tab:exp_divergence}, we exclude the pretraining dataset for direct comparison.
We evaluate WASD on three model families: GPT-2~\cite{radford2019language}, OpenLLaMA2~\cite{touvron2023llama2}, and Qwen2.5~\cite{hui2024qwen2}. For GPT-2, we employ GPT-2 XL (1.5B) as the teacher and GPT-2 Base (0.1B) and Medium (0.3B) as students. For OpenLLaMA2, we use OpenLLaMA2-7B as the teacher and OpenLLaMA2-3B as the student, trained with LoRA~\cite{hu2022lora}.
For Qwen2.5, we use Qwen2.5-7B-Instruct as the teacher and Qwen2.5-1.5B-Instruct as the student, following the distillation setups of previous work~\cite{kim2026distillation,ko2025distillm2}.

We compare WASD against baselines using different divergence objectives. Baseline details are provided in \cref{app_subsec:baselines}.
For GPT-2 and OpenLLaMA2, we evaluate on five instruction-following datasets: (1) the Dolly evaluation dataset, (2) Self-Instruct~\cite{wang-etal-2023-self-instruct}, (3) Vicuna~\cite{vicuna2023}, (4) Super-Natural Instructions (Super NI)~\cite{wang-etal-2022-super}, and (5) Unnatural Instructions (UnNI)~\cite{honovich-etal-2023-unnatural}. We report ROUGE-L~\cite{lin-2004-rouge} to measure generation quality through similarity to reference text and Self-BLEU~\cite{zhu2018texygen} to measure generation diversity.
For Qwen2.5, we evaluate on AlpacaEval~\cite{alpaca_eval}, Evol-Instruct~\cite{xu2024wizardlm}, and UltraFeedback~\cite{cui2024ultrafeedback}, reporting win rates against reference responses using GPT-4o-mini as the judge.


\paragraph{Comparison to other divergences}
We first conduct a divergence-level comparison to examine whether the proposed Wasserstein-based objective itself provides a better distillation signal. In this comparison, we do not use an additional pretraining loss, SFT initialization for the student, or on-policy techniques. \cref{tab:exp_divergence} shows that WASD achieves the best average performance among the compared divergences, demonstrating the effectiveness of the proposed divergence. We further analyze the quality-diversity trade-off by varying the decoding temperature. In \cref{subfig:motivation_trade_off}, WASD consistently forms a favorable frontier over existing baselines, indicating that WASD preserves generation quality even with increased diversity. This behavior is likely due to the token-embedding transport cost used in WASD: by assigning lower costs to tokens with similar meanings, WASD can preserve generation quality even when decoding is encouraged to produce more diverse outputs.

\paragraph{Compatibility with recent distillation frameworks}
We evaluate whether this benefit remains when combined with recent distillation methods. Following the setup of AMiD~\cite{shin2026amid}, we additionally incorporate the assistant distribution proposed in AMiD and compare WASD against recent baselines under the same training framework. As shown in \cref{tab:main}, WASD consistently achieves the best average performance across different GPT-2 student scales and the larger OpenLLaMA2-7B teacher setting. These results suggest that the token-embedding cost structure in WASD contributes additional improvements beyond those achieved by recent assistant-distribution-based distillation methods.

\begin{wrapfigure}[14]{R}{0.50\textwidth}
    \vspace{-.5em}
    \centering
    \begin{subfigure}[b]{0.49\linewidth}
        \centering
        \includegraphics[width=\linewidth]{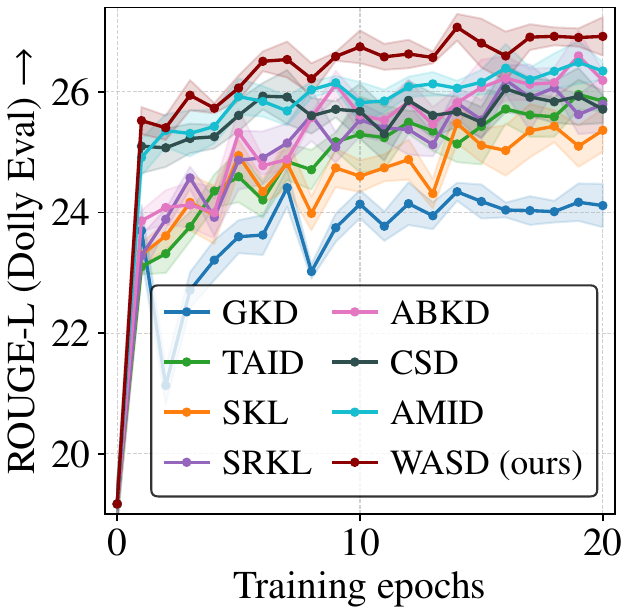}
        \caption{}
        \label{subfig:training_curve_epoch}
    \end{subfigure}
    \hfill 
    \begin{subfigure}[b]{0.49\linewidth}
        \centering
        \includegraphics[width=\linewidth]{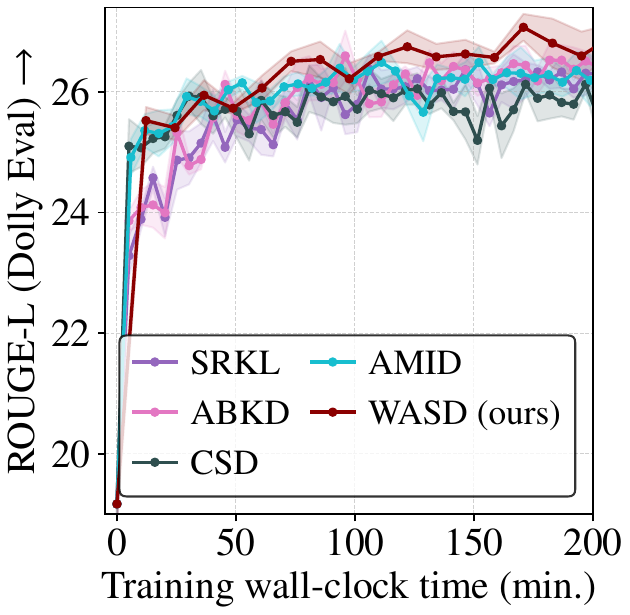}
        \caption{}
        \label{subfig:training_curve_time}
    \end{subfigure}
    \caption{ROUGE-L during training on Dolly Eval.}
    \label{fig:curve}
\end{wrapfigure}

\cref{fig:curve} illustrates the training dynamics of WASD. \cref{subfig:training_curve_epoch} shows test performance across training epochs, where WASD consistently maintains higher performance than competing methods throughout most of training. Since WASD additionally requires matrix scaling to compute the dual variables $\bm{\phi}$, it requires approximately twice the training time and 20\% more training GPU memory than the compared baselines in our experiments, with no additional inference-time overhead.
To compare performance under the same time budget, we extend the training epochs of the competing baselines and report performance against wall-clock time in \cref{subfig:training_curve_time}. Even under this setting, WASD continues to outperform competing methods at the same training time. These results show that faster convergence can offset the additional per-epoch cost when targeting a given performance level.

\paragraph{LLM-based evaluation}
We additionally use LLM-based judgments to assess response quality beyond ROUGE-L, which primarily measures lexical overlap with reference answers. For GPT-2, we follow the GPT-4 feedback protocol of prior distillation studies~\cite{gu2024minillm,kim2026distillation,ko2024distillm} under the same experimental setting as \cref{tab:main}. GPT-4 assigns a score from 1 to 10 to each model response and reference answer, and we report the ratio of total model-response score to the total reference-answer score as a percentage. As shown in \cref{tab:gpt4_feedback}, WASD outperforms ABKD, CSD, and AMiD on all three benchmarks. These results suggest that incorporating token-level semantic information may improve response quality as assessed by GPT-4 feedback, which can account for semantic similarity between generated responses and reference answers.

For Qwen2.5, we combine the DistiLLM-2 training framework with the AMiD assistant distribution and the forward KL objective, denoting this baseline as AMiD. WASD replaces the forward KL term with the proposed Sinkhorn divergence objective while keeping the remaining training configuration unchanged. \cref{tab:qwen25} shows that WASD achieves higher win rates than AMiD across all three benchmarks under GPT-4o-mini evaluation. These results further support the effectiveness of the proposed objective for instruction-following distillation between Qwen2.5 models.

\begin{figure}[t]
    \centering
    \begin{minipage}[t]{0.46\textwidth}
        \centering
        \vspace*{0pt}
        \captionof{table}{GPT-4 feedback scores for \texttt{GPT-2 XL (1.5B)} $\rightarrow$ \texttt{GPT-2 Base (0.1B)}, reported as percentages of reference-answer scores.}
        \resizebox{\textwidth}{!}{%
        \begin{tabular}{lccc}
            \toprule
            Method & Dolly Eval & Self Inst & Vicuna \\
            \midrule
            ABKD~\cite{wang2025abkd} & 36.09 & 20.30 & 24.67 \\
            CSD~\cite{kim2026distillation} & 35.32 & 19.91 & 22.40 \\
            AMiD~\cite{shin2026amid} & 36.35 & 18.87 & 22.00 \\
            \textbf{WASD (ours)} & \textbf{37.84} & \textbf{20.90}  & \textbf{25.07}  \\
            \bottomrule
        \end{tabular}%
        }
        \label{tab:gpt4_feedback}
    \end{minipage}
    \hfill
    \begin{minipage}[t]{0.52\textwidth}
        \centering
        \vspace*{0pt}
        \captionof{table}{Instruction-following win rates (\%) against reference responses for \texttt{Qwen2.5-7B-Instruct} $\rightarrow$ \texttt{Qwen2.5-1.5B-Instruct}, using GPT-4o-mini.}
        \resizebox{\textwidth}{!}{%
        \begin{tabular}{lccc}
            \toprule
            Method & AlpacaEval & Evol-Instruct & UltraFeedback \\
            \midrule
            Teacher & 95.31 & 94.50 & 84.16 \\
            Student & 71.80 & 58.03 & 45.06 \\
            \midrule
            AMiD~\cite{shin2026amid} & 89.29 & 83.94 & 72.81 \\
            \textbf{WASD (ours)} & \textbf{90.04} & \textbf{85.32}  & \textbf{73.21}  \\
            \bottomrule
        \end{tabular}%
        }
        \label{tab:qwen25}
    \end{minipage}
    \vspace{-.1em}
\end{figure}

\begin{figure}[tp]   
    \centering
    \begin{minipage}[t]{0.50\textwidth}
        \vspace*{0pt}
        \centering
        \captionof{table}{Effect of the transport cost on distillation performance for \texttt{GPT-2 XL (1.5B)} $\rightarrow$ \texttt{GPT-2 Base (0.1B)}. Avg. ROUGE-L denotes the average over five instruction-following benchmarks. Results are averaged over three training seeds.}
        \resizebox{0.95\textwidth}{!}{%
        \begin{tabular}{lc}
            \toprule
            Transport cost & Avg. ROUGE-L ($\uparrow$) \\
            \midrule
            AMiD~\cite{shin2026amid} & \rateinline{23.36}{0.07} \\
            \midrule
            Uniform cost & \rateinline{22.08}{0.07} \\
            Permuted cost & \rateinline{23.47}{0.26} \\
            \textbf{Semantic cost (WASD)} & \rateinline{\textbf{23.94}}{0.15} \\
            \bottomrule
        \end{tabular}%
        }
        \label{tab:transport_cost}
    \end{minipage}
    \hfill
    \begin{minipage}[t]{0.48\textwidth}
        \vspace*{0pt}

        \refstepcounter{figure}
        \setcounter{subfigure}{0}
        
        \begin{subfigure}[b]{0.48\linewidth}
            \centering
            \includegraphics[width=\linewidth]{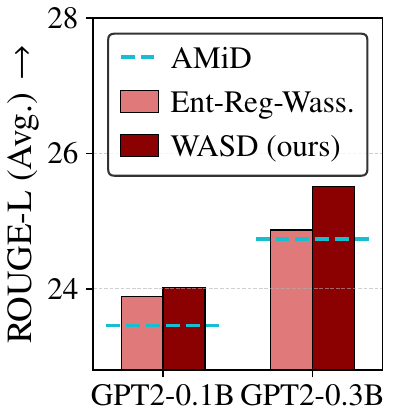}
            \caption{Performance of the entropy-regularized Wasserstein distance.}
            \label{subfig:abl_entwass}
        \end{subfigure}
        \hfill 
        \begin{subfigure}[b]{0.48\linewidth}
            \centering
            \includegraphics[width=\linewidth]{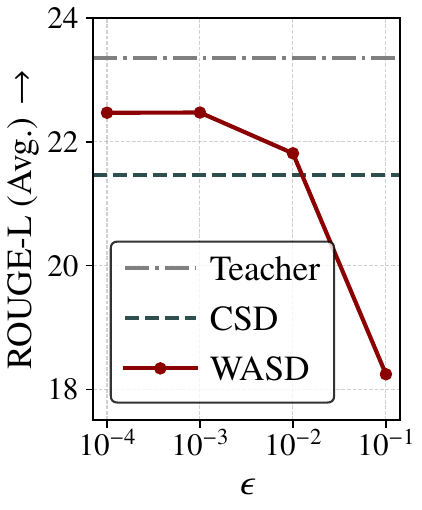}
            \caption{Sensitivity analysis of hyperparameter $\epsilon$.}
            \label{subfig:abl_eps}
        \end{subfigure}
        \addtocounter{figure}{-1}
        \caption{Ablation studies of WASD.}
        \label{fig:abl}
    \end{minipage}
    \vspace{-1em}
\end{figure}

\paragraph{Ablation studies}
We first examine whether the gains of WASD depend on meaningful token geometry through a controlled transport-cost ablation. We follow the GPT-2 XL (1.5B) to Base (0.1B) setup in \cref{tab:main}, but average results over three training seeds, which accounts for the differences from the reported values. We compare three cost matrices while keeping all other training settings fixed: a uniform cost that assigns a constant transport cost, a permuted cost that randomly reassigns token relations, and a semantic cost based on cosine distances between teacher token embeddings. The permuted cost retains Sinkhorn transport while removing meaningful token correspondence, whereas the uniform cost removes graded token geometry. As shown in \cref{tab:transport_cost}, the semantic cost achieves the highest average ROUGE-L. Permuting token relations reduces performance to a level comparable to AMiD, while the uniform cost performs below the baseline. These results suggest that meaningful token geometry contributes to the gains of WASD, rather than generic smoothing alone.

We next analyze the effect of correcting entropic bias. As discussed in \cref{subsec:method_objective}, entropy-regularized Wasserstein distance introduces an entropic bias, which Sinkhorn divergence mitigates by subtracting the self-similarity terms. To examine this effect, we evaluate entropy-regularized Wasserstein distance under the GPT-2 distillation setup in \cref{tab:main}, and we report average ROUGE-L scores in \cref{subfig:abl_entwass}. Although entropy-regularized Wasserstein distance still outperforms the strongest baseline, AMiD, likely due to its token-embedding transport cost, it underperforms WASD. This result suggests that correcting entropic bias helps better utilize semantically informed transport costs in distillation.

We also conduct a sensitivity analysis on the entropy regularization hyperparameter $\epsilon$ under the setting of \cref{tab:exp_divergence}, with results in \cref{subfig:abl_eps}. When $\epsilon$ is sufficiently small, WASD consistently outperforms the strongest baseline, CSD. However, performance gradually decreases as $\epsilon$ increases. This behavior is consistent with the transport kernel $\mK = \exp (-\mC / \epsilon)$ becoming less sensitive to differences in the transport cost matrix $\mC$, weakening the influence of the token-embedding cost structure. The toy visualization in \cref{app_fig:loss_sink_eps} of \cref{app_sec:exp_results} shows that the loss landscape becomes flatter as $\epsilon$ increases, suggesting reduced sensitivity to deviations from the teacher distribution in this example.


\subsection{Task-specific experiments}
\label{subsec:exp_specific}

\paragraph{Setup}
Following previous work~\cite{ko2025distillm2,shin2026amid,xu2025speculative}, we utilize diverse task-specific distillation datasets: Flores-200~\cite{costa2022no} for translation, DialogSum~\cite{chen-etal-2021-dialogsum} for summarization, GSM8k~\cite{cobbe2021training} for arithmetic reasoning, and WizardCoder~\cite{luo2023wizardcoder} for code generation.
We employ the supervised fine-tuned Gemma-7B-IT~\cite{team2024gemma} as the teacher and Gemma-2B-IT as the student. For the code generation task, we use Qwen2.5-Coder-7B-Instruct~\cite{hui2024qwen2} as the teacher and Qwen2.5-Coder-1.5B-Instruct as the student.
We evaluate on the test dataset for each task, using HumanEval~\cite{chen2021evaluating} and MBPP~\cite{austin2021program} for code generation. We report COMET~\cite{rei2022comet} for translation, ROUGE-L for summarization, accuracy for arithmetic reasoning, and pass@1 for code generation.



\paragraph{Results}
As shown in \cref{tab:task_specific,tab:task_specific_code}, WASD consistently improves performance on translation, summarization, arithmetic reasoning, and code generation benchmarks across different model series and task domains. These results demonstrate that incorporating token-level semantic information provides robust benefits even in large-scale and specialized distillation settings. In the current implementation, we construct the cost matrix using the token embedding space of the teacher model, which is the same as in the task-agnostic distillation setting. However, the flexibility of the Wasserstein formulation allows the cost matrix to capture richer structures beyond general token semantics. In particular, incorporating task-specific or domain knowledge into the cost matrix may further improve distillation performance, which we leave for future work.

\begin{figure}[t]
    \centering
    \begin{minipage}[t]{0.52\textwidth}
        \centering
        \vspace*{0pt}
        \captionof{table}{Task-specific distillation results on translation, summarization, and arithmetic reasoning for fine-tuned \texttt{Gemma-7B-IT} $\rightarrow$ \texttt{Gemma-2B-IT}.}
        \resizebox{\textwidth}{!}{%
        \begin{tabular}{lccc}
            \toprule
             & Translation & Summarization & Arithmetic \\
            Method & COMET ($\uparrow$) & ROUGE-L ($\uparrow$) & Accuracy ($\uparrow$) \\
            \midrule
            Teacher & 79.33 & 36.05 & 55.04 \\
            Student & 44.39 & 11.70 & 0.00 \\
            \midrule
            CSD~\cite{kim2026distillation} & \underline{73.71} & 34.88 & 23.12 \\
            AMiD~\cite{shin2026amid} & 73.70 & \underline{35.09} & \underline{24.49} \\
            \textbf{WASD (ours)} & \textbf{74.53} & \textbf{35.15}  & \textbf{24.87}  \\
            \bottomrule
        \end{tabular}%
        }
        \label{tab:task_specific}
    \end{minipage}
    \hfill
    \begin{minipage}[t]{0.46\textwidth}
        \vspace*{0pt}
        \centering
        \captionof{table}{Distillation results on code generation for \texttt{Qwen2.5-Coder-Instruct} 7B $\rightarrow$ 1.5B, evaluated with pass@1.}
        \resizebox{\textwidth}{!}{%
        \begin{tabular}{lccc}
            \toprule
            Method & HumanEval& MBPP& Avg. ($\uparrow$) \\
            \midrule
            Teacher & 91.5 & 82.3 & 86.9 \\
            Student & 69.5 & 68.8 & 69.2 \\
            \midrule
            DistiLLM-2~\cite{ko2025distillm2} & 72.6 & 72.5 & 72.6 \\ 
            CSD~\cite{kim2026distillation} & 73.2 & 73.0 & 73.1 \\ 
            AMiD~\cite{shin2026amid} & \underline{73.8} & \underline{74.3} & \underline{74.1} \\
            \textbf{WASD (ours)} & \textbf{74.4} & \textbf{75.4} & \textbf{74.9} \\
            \bottomrule
        \end{tabular}%
        }
        \label{tab:task_specific_code}
    \end{minipage}
\end{figure}

\section{Conclusion}
\label{sec:conc}

We propose WASD, a Wasserstein-based knowledge distillation method for large language models (LLMs) that incorporates token-level semantic information through a transport-based formulation. By adopting the Sinkhorn divergence and deriving an efficient objective based on optimal transport dual variables, WASD enables practical semantic-aware distillation without additional networks or backpropagation through the transport solver. Experiments across multiple model families and tasks show that WASD consistently improves distillation performance over existing distillation methods, particularly in quality-diversity trade-offs. These results highlight the importance of token-level information encoded in the vocabulary space for effective knowledge transfer in LLM distillation and suggest promising future directions for semantic-aware distribution alignment.

\newpage

\begin{ack}

This work was supported by the InnoCORE program of the Ministry of Science and ICT (N10260008) (10\%). This work was supported by the IITP (Institute of Information \& Communications Technology Planning \& Evaluation)-ITRC (Information Technology Research Center) grant funded by the Korea government (Ministry of Science and ICT) (IITP-2026-RS-2024-00437268) (45\%). This research was supported by AI Technology Development for Commonsense Extraction, Reasoning, and Inference from Heterogeneous Data (IITP) funded by the Ministry of Science and ICT (RS-2022-II220077) (45\%).

\end{ack}

\bibliography{WASD}
\bibliographystyle{plainnat}

\newpage


\appendix

\crefalias{section}{appendix}
\crefalias{subsection}{appendix}
\crefname{appendix}{Appendix}{Appendices}
\Crefname{appendix}{Appendix}{Appendices}

\section{Proofs and derivations}
\label{app_sec:proof}

\subsection{\texorpdfstring{Proof of \cref{thm:entW_objective}}{Proof of Lemma 1}}
\label{app_subsec:proof_entW}

\lema*
\begin{proof}

By strong duality, the entropy-regularized Wasserstein distance ${\tilde{D}}_{{\text{W}}}^{\epsilon} (p, q_{\bm{\theta}})$ is equal to the optimal value of the following dual problem:
\begin{align}
    \max_{\bm{\phi},\bm{\psi}} \sum_{i=1}^{d} \phi_i q_{\bm{\theta}}(y_l=i|\rvx,\rvy_{<l}) + \sum_{j=1}^{d} \psi_j p(y_l=j|\rvx,\rvy_{<l}) - \sum_{i=1}^{d} \sum_{j=1}^{d} \epsilon \exp \big (\frac{1}{\epsilon} (\phi_i + \psi_j - C_{ij} ) \big ).\label{app_eq:eot_dual}
\end{align}
Substituting the optimal dual variables $\bm{\phi}^{*(p,q_{\bm{\theta}})}(\rvx,\rvy_{<l})$ and $\bm{\psi}^{*(p,q_{\bm{\theta}})}(\rvx,\rvy_{<l})$ into the dual objective yields
\begin{align}
 & {\tilde{D}}_{{\text{W}}}^{\epsilon} (p(y_l|\rvx,\rvy_{<l}), q_{\bm{\theta}} (y_l|\rvx,\rvy_{<l})) \nonumber \\
 &=   \sum_{i=1}^{d} \phi_i^{*(p,q_{\bm{\theta}})}(\rvx,\rvy_{<l}) q_{\bm{\theta}}(y_l=i|\rvx,\rvy_{<l}) + \sum_{j=1}^{d} \psi_j^{*(p,q_{\bm{\theta}})} (\rvx,\rvy_{<l})p(y_l=j|\rvx,\rvy_{<l})  \nonumber \\
 & \quad - \sum_{i=1}^{d} \sum_{j=1}^{d} \epsilon \exp \big (\frac{1}{\epsilon} (\phi_i^{*(p,q_{\bm{\theta}})}(\rvx,\rvy_{<l}) + \psi_j^{*(p,q_{\bm{\theta}})}(\rvx,\rvy_{<l}) - C_{ij} ) \big ). \label{app_eq:entw_dual_value}
\end{align}
By the optimality conditions in \cref{eq:optimal_sol}, the optimal transport plan satisfies
\begin{align}
    {P_{ij}^{(l)}}^{*} = \exp \left ( \frac{1}{\epsilon} (\phi_i^{*(p,q_{\bm{\theta}})}(\rvx,\rvy_{<l}) + \psi_j^{*(p,q_{\bm{\theta}})}(\rvx,\rvy_{<l}) - \emC_{ij} )\right ).
\end{align}
Since ${\mP^{(l)}}^{*} \in U_l (p,q_{\bm{\theta}})$, its total mass is $1$. Therefore, the last term in \cref{app_eq:entw_dual_value} equals $-\epsilon$, giving
\begin{align}
 & {\tilde{D}}_{{\text{W}}}^{\epsilon} (p(y_l|\rvx,\rvy_{<l}), q_{\bm{\theta}} (y_l|\rvx,\rvy_{<l})) \nonumber \\
 = &  \sum_{i=1}^{d} \phi_i^{*(p,q_{\bm{\theta}})}(\rvx,\rvy_{<l}) q_{\bm{\theta}}(y_l=i|\rvx,\rvy_{<l}) + \sum_{j=1}^{d} \psi_j^{*(p,q_{\bm{\theta}})} (\rvx,\rvy_{<l})p(y_l=j|\rvx,\rvy_{<l})  - \epsilon.
\end{align}

\end{proof}

\subsection{\texorpdfstring{Proof of \cref{thm:gradient}}{Proof of Proposition 2}}
\label{app_subsec:proof_gradient}

\gradient*
\begin{proof}
For a fixed context $(\rvx,\rvy_{<l})$, assume that $\epsilon >0$, the cost matrix $\mC$ is finite and fixed, the teacher distribution is fixed, and both distributions are strictly positive, with $q_{\bm{\theta}}$ differentiable in $\bm{\theta}$.
By strong duality, the entropy-regularized Wasserstein distance is equal to the optimal value of the dual problem in \cref{app_eq:eot_dual}. Specifically, define the dual objective as follows: 
\begin{align}
    \mathcal{L}(\bm{\theta},\bm{\phi},\bm{\psi}) := & \sum_{i=1}^{d} \phi_i q_{\bm{\theta}}(y_l=i|\rvx,\rvy_{<l}) + \sum_{j=1}^{d} \psi_j p(y_l=j | \rvx,\rvy_{<l}) \nonumber \\
    & - \epsilon \sum_{i=1}^{d} \sum_{j=1}^{d} \exp \left ( \frac{1}{\epsilon} ( \phi_i + \psi_j - C_{ij} ) \right ).
\end{align}
Then, we have
\begin{align}
    \tilde{D}_{\text{W}}^{\epsilon} (p(y_l|\rvx,\rvy_{<l}),q_{\bm{\theta}}(y_l|\rvx,\rvy_{<l})) = \max_{\bm{\phi},\bm{\psi}} \mathcal{L} (\bm{\theta}, \bm{\phi}, \bm{\psi}).
\end{align}

Although the optimal dual variables depend implicitly on $\bm{\theta}$ through $q_{\bm{\theta}}$, the envelope theorem~\cite{milgrom2002envelope} allows us to differentiate the optimal value by treating the optimal dual variables as fixed. In particular, only the student marginal term in $\mathcal{L}$ depends explicitly on $\bm{\theta}$. Therefore,
\begin{align}
    \nabla_{\bm{\theta}} \tilde{D}_{\text{W}}^{\epsilon} ( p(y_l | \rvx, \rvy_{<l}) , q_{\bm{\theta}} (y_l | \rvx,\rvy_{<l}) ) & =\left.
    \nabla_{\bm{\theta}}\mathcal{L}(\bm{\theta},\bm{\phi},\bm{\psi})
    \right|_{
        \bm{\phi}=\bm{\phi}^{*(p,q_{\bm{\theta}})},
        \bm{\psi}=\bm{\psi}^{*(p,q_{\bm{\theta}})}
    } \\
    & = \sum_{i=1}^{d} {\phi}_i^{*(p,q_{\bm{\theta}})} (\rvx,\rvy_{<l}) \nabla_{\bm{\theta}} q_{\bm{\theta}}(y_l=i|\rvx,\rvy_{<l}).
\end{align}
Note that the additive ambiguity of the dual potentials does not affect this gradient, since $\sum_{i=1}^{d}\nabla_{\bm{\theta}}
q_{\bm{\theta}}(y_l=i|\rvx,\rvy_{<l})=0$.
\end{proof}

\subsection{\texorpdfstring{Proof of \cref{thm:sink_objective}}{Proof of Theorem 3}}
\label{app_subsec:proof_sink}

\thma*
\begin{proof}

By definition, the Sinkhorn divergence is
\begin{align}
    {D}_{{\text{S}}}^{\epsilon}( p ,  q_{\bm{\theta}} ) := {\tilde{D}}_{{\text{W}}}^{\epsilon}( p,  q_{\bm{\theta}})
    - \frac{1}{2} {\tilde{D}}_{{\text{W}}}^{\epsilon}( q_{\bm{\theta}},  q_{\bm{\theta}})
    - \frac{1}{2} {\tilde{D}}_{{\text{W}}}^{\epsilon}( p, p).\label{app_eq:sinkhorn}
\end{align}
We first apply \cref{thm:gradient} to the first term ${\tilde{D}}_{{\text{W}}}^{\epsilon}( p,  q_{\bm{\theta}})$. This gives
\begin{align}
    \nabla_{\bm{\theta}}\tilde{D}_{\text{W}}^{\epsilon}(p(y_l|\rvx,\rvy_{<l}),q_{\bm{\theta}}(y_l|\rvx,\rvy_{<l})) = \sum_{i=1}^{d} \phi_i^{*(p,q_{\bm{\theta}})}(\rvx,\rvy_{<l}) \nabla_{\bm{\theta}}q_{\bm{\theta}}(y_l=i|\rvx,\rvy_{<l}).
    \label{app_eq:sink1}
\end{align}

Next, consider the second term ${\tilde{D}}_{{\text{W}}}^{\epsilon}( q_{\bm{\theta}},  q_{\bm{\theta}})$. Applying the same argument as in the proof of \cref{thm:gradient}, both marginal terms in the dual objective depend explicitly on $\bm{\theta}$. Therefore,
\begin{align}
    &\nabla_{\bm{\theta}}\tilde{D}_{\text{W}}^{\epsilon} (q_{\bm{\theta}}(y_l|\rvx,\rvy_{<l}), q_{\bm{\theta}}(y_l|\rvx,\rvy_{<l})) \nonumber \\
    &=\sum_{i=1}^{d} \phi_i^{*(q_{\bm{\theta}},q_{\bm{\theta}})}(\rvx,\rvy_{<l}) \nabla_{\bm{\theta}}q_{\bm{\theta}}(y_l=i|\rvx,\rvy_{<l})+\sum_{i=1}^{d}\psi_i^{*(q_{\bm{\theta}},q_{\bm{\theta}})}(\rvx,\rvy_{<l})\nabla_{\bm{\theta}}q_{\bm{\theta}}(y_l=i|\rvx,\rvy_{<l}).\label{app_eq:sink2}
\end{align}

In the self-transport case, the two marginals are identical. By symmetry of the cost matrix, the optimal dual variables $\bm{\phi}^{*(q_{\bm{\theta}},q_{\bm{\theta}})}$ and $\bm{\psi}^{*(q_{\bm{\theta}},q_{\bm{\theta}})}$ can be chosen equal. We take the representative satisfying $\bm{\phi}^{*(q_{\bm{\theta}},q_{\bm{\theta}})} = \bm{\psi}^{*(q_{\bm{\theta}},q_{\bm{\theta}})}$. Therefore, 
\begin{align}
    &\nabla_{\bm{\theta}}\tilde{D}_{\text{W}}^{\epsilon} (q_{\bm{\theta}}(y_l|\rvx,\rvy_{<l}), q_{\bm{\theta}}(y_l|\rvx,\rvy_{<l}))=2\sum_{i=1}^{d} \phi_i^{*(q_{\bm{\theta}},q_{\bm{\theta}})}(\rvx,\rvy_{<l})\nabla_{\bm{\theta}}q_{\bm{\theta}}(y_l=i|\rvx,\rvy_{<l}). \label{app_eq:sink3}
\end{align}

Finally, the last term ${\tilde{D}}_{{\text{W}}}^{\epsilon}( p, p)$ does not contain the student distribution and is therefore independent of $\bm{\theta}$. Combining these gradients with the coefficients in \cref{app_eq:sinkhorn}, we have
\begin{align}
    \nabla_{\bm{\theta}}D_{\text{S}}^{\epsilon}(p(y_l|\rvx,\rvy_{<l})&,q_{\bm{\theta}}(y_l|\rvx,\rvy_{<l}))\nonumber\\
    &=\sum_{i=1}^{d}\Big( \phi_i^{*(p,q_{\bm{\theta}})}(\rvx,\rvy_{<l})- \phi_i^{*(q_{\bm{\theta}},q_{\bm{\theta}})}(\rvx,\rvy_{<l})\Big)\nabla_{\bm{\theta}}q_{\bm{\theta}}(y_l=i|\rvx,\rvy_{<l}).\label{app_eq:sink4}
\end{align}

Now, we differentiate the WASD loss in \cref{eq:KD_wasd}. The stop-gradient operator preserves the potential difference in the forward pass and treats it as constant during backpropagation. Under the stop-gradient convention, we obtain
\begin{align}
    &\nabla_{\bm{\theta}} \mathcal{L}_{\text{WASD}}^{\epsilon}(\bm{\theta}) \nonumber\\
    &= \mathbb{E}_{(\rvx,\rvy)\sim\mathcal{D}}  \Bigg[  \sum_{l=1}^{L}\sum_{i=1}^{d}  \nabla_{\bm{\theta}} \Bigg\{ \operatorname{sg}\!\Big( \phi_i^{*(p,q_{\bm{\theta}})}(\rvx,\rvy_{<l}) - \phi_i^{*(q_{\bm{\theta}},q_{\bm{\theta}})}(\rvx,\rvy_{<l}) \Big) q_{\bm{\theta}}(y_l=i|\rvx,\rvy_{<l})  \Bigg\}   \Bigg] \label{app_eq:sink_last1}  \\
    &=  \mathbb{E}_{(\rvx,\rvy)\sim\mathcal{D}}\Bigg[  \sum_{l=1}^{L}\sum_{i=1}^{d} \operatorname{sg}\!\Big(  \phi_i^{*(p,q_{\bm{\theta}})}(\rvx,\rvy_{<l})  -\phi_i^{*(q_{\bm{\theta}},q_{\bm{\theta}})}(\rvx,\rvy_{<l}) \Big) \nabla_{\bm{\theta}}q_{\bm{\theta}}(y_l=i|\rvx,\rvy_{<l}) \Bigg] \label{app_eq:sink_last2} \\
    &=\mathbb{E}_{(\rvx,\rvy)\sim\mathcal{D}}  \Bigg[  \sum_{l=1}^{L}\sum_{i=1}^{d} \Big( \phi_i^{*(p,q_{\bm{\theta}})}(\rvx,\rvy_{<l})  - \phi_i^{*(q_{\bm{\theta}},q_{\bm{\theta}})}(\rvx,\rvy_{<l}) \Big) \nabla_{\bm{\theta}}q_{\bm{\theta}}(y_l=i|\rvx,\rvy_{<l})
    \Bigg] \label{app_eq:sink_last3} \\
    &= \mathbb{E}_{(\rvx,\rvy)\sim\mathcal{D}} \Bigg[  \sum_{l=1}^{L} \nabla_{\bm{\theta}}D_{\text{S}}^{\epsilon}\Big( p(y_l|\rvx,\rvy_{<l}), q_{\bm{\theta}}(y_l|\rvx,\rvy_{<l}) \Big) \Bigg] \label{app_eq:sink_last4} \\
    &= \nabla_{\bm{\theta}}\mathcal{L}_{\text{S}}^{\epsilon}(\bm{\theta}). \label{app_eq:sink_fin}
\end{align}
Here, \cref{app_eq:sink_last2} follows from the stop-gradient rule, \cref{app_eq:sink_last3} uses that $\operatorname{sg}$ preserves its input value, and \cref{app_eq:sink_last4} follows from \cref{app_eq:sink4}.
\end{proof}

\section{Additional discussion on related work}
\label{app_sec:related_work}



\subsection{Wasserstein distance and Sinkhorn divergence}
\label{app_subsec:sinkhorn_div}

The Wasserstein distance provides a geometry-aware discrepancy between probability distributions by measuring the minimum cost of transporting probability mass across the underlying space. Its practical use in machine learning has been enabled by entropy-regularized optimal transport and Sinkhorn iterations~\cite{cuturi2013sinkhorn,cuturi2014fast,peyre2019computational}.
Wasserstein-based distances have also been widely adopted across machine learning applications. In generative modeling, WGANs use the Wasserstein distance between the generator and data distributions to improve training stability and mitigate mode collapse~\cite{arjovsky2017wasserstein}. In reinforcement learning, Wasserstein natural gradient methods use the geometry induced by Wasserstein penalties to define efficient policy updates over behavioral distributions~\cite{moskovitz2020efficient}. More recently, Wasserstein and Sinkhorn trust-region policy optimization methods have been studied as alternatives to KL-based trust regions, deriving policy updates from Lagrangian duality and exploiting optimal-transport geometry in policy space~\cite{song2023provably}. 

Sinkhorn divergence was introduced as a scalable objective for large-scale generative modeling, combining entropy-regularized OT, Sinkhorn fixed-point iterations, self-similarity correction, and GPU-based automatic differentiation, while showing that the regularization parameter interpolates between Wasserstein distance and maximum mean discrepancy (MMD)~\cite{genevay2018learning}. 
Further theoretical analysis established key properties of Sinkhorn divergence, including positivity, convexity, and metrization of convergence in law, providing theoretical support for its use as a statistical discrepancy~\cite{feydy2019interpolating}. Sinkhorn-based objectives have since been extended to various distribution-alignment settings, further demonstrating their applicability beyond generative modeling~\cite{sanjabi2018convergence}.

\paragraph{Comparison with WPR~\cite{na2026semanticaware}}
WPR is closely related to our work in its use of token-space optimal transport. It replaces the KL regularizer in reinforcement learning from human preferences (RLHF) with an entropy-regularized Wasserstein distance over token distributions. WASD builds on WPR's token-embedding transport costs, nearest-$k$ cost-matrix truncation, and computation of optimal dual potentials through Sinkhorn iterations, extending this framework to teacher-student distribution matching.
The two methods differ in the role of the transport objective. WPR uses entropy-regularized Wasserstein distance to regularize a policy against a reference model during reward maximization. In KD, the teacher-student discrepancy serves as the primary training signal and should be minimized when the student distribution matches the teacher distribution. As discussed in \cref{subsec:method_objective}, entropic bias can prevent entropy-regularized Wasserstein distance from satisfying this requirement. Therefore, WASD adopts the Sinkhorn divergence, which corrects this bias by subtracting the teacher and student self-transport terms.
This difference also leads to distinct optimization procedures. WPR evaluates the cross-transport potential at sampled tokens and incorporates it as a scalar reward penalty during rollout-based RL optimization. WASD directly weights student probabilities over the vocabulary by the difference between teacher-student and student self-transport potentials. The student self-transport potential is computed through fixed-point iterations, and the resulting stop-gradient objective yields the Sinkhorn divergence gradient without backpropagating through the transport solver, as established in \cref{thm:sink_objective}. The ablation in \cref{subfig:abl_entwass} compares entropy-regularized Wasserstein distance with Sinkhorn divergence under identical distillation settings, showing the additional benefit of the self-similarity correction.



\subsection{Wasserstein distance in knowledge distillation}
\label{app_subsec:WD_KD}

Optimal transport (OT) has been explored for knowledge distillation before the recent interest in LLM distillation, mainly for representation- or feature-level transfer. WCoRD~\cite{chen2021wasserstein} leverages both the dual and primal forms of the Wasserstein distance: the dual form is used to construct a global contrastive objective for transferring teacher-student representation knowledge, while the primal form regularizes local mini-batch feature distribution matching. \citet{lohit2022model} similarly formulates model compression as minimizing optimal-transport costs between teacher and student intermediate feature distributions, computing transport costs during student training. In NLP, KNOT~\cite{bhardwaj2022knot} applies optimal transport to multi-teacher distillation by minimizing the transport cost between the student's label distribution and a weighted aggregation of teacher label distributions. These works demonstrate the usefulness of OT for transferring structured knowledge.

Several recent works have explored Wasserstein-based distances more directly as distillation discrepancies. SinKD~\cite{cui2024sinkhorn} applies the entropy-regularized Wasserstein distance to logit-based distillation and introduces a batch-wise formulation to capture geometric structure across samples, showing improvement over several $f$-divergence-based objectives on NLP classification benchmarks. However, its transport cost is defined over sample-level geometry, and it does not model semantic relations among vocabulary tokens. In computer vision, WKD~\cite{lv2024wasserstein} uses discrete Wasserstein distance for logit distillation by exploiting category interrelations, and continuous Wasserstein distance for feature distillation by matching Gaussian feature distributions. While conceptually related, WKD targets class-level and feature-level transfer in vision models rather than autoregressive next-token distributions in LLMs.

Another line of work uses optimal transport to address cross-tokenizer distillation, where teacher and student vocabularies differ. ULD~\cite{boizard2025towards} introduces a Wasserstein-based universal logit distillation loss and derives an efficient closed form under uniform-support and uniform-cost assumptions. MultiLevelOT~\cite{cui2025multi} further combines token-level and sequence-level optimal transport with multiple cost matrices to align cross-tokenizer logit distributions. MCW-KD~\cite{vuong2026mcw} extends this direction by optimizing multiple Wasserstein costs over hidden states and output distributions. These methods primarily address support mismatch or heterogeneous model alignment, whereas our setting assumes a shared vocabulary and focuses on the semantic geometry within that vocabulary. Therefore, these methods are complementary to our work: they relax the need for
a common vocabulary, while WASD studies what can be gained when a common
vocabulary is available and its embedding geometry can be used as an informative
transport cost.

A further distinction lies in the objective and its optimization form of the transport discrepancy. Prior OT-based KD methods typically use the OT solver to evaluate a primal transport cost, such as $\langle \mP, \mC \rangle$. For example, SinKD computes an entropy-regularized transport plan and minimizes its transport cost, WKD-L solves an entropy-regularized discrete WD objective, and ULD avoids Sinkhorn iterations through a closed form induced by a uniform-cost assumption. However, entropy-regularized OT introduces a self-similarity bias: the regularized transport cost between identical distributions is not zero, and minimizing this biased objective alone does not necessarily make the student distribution match the teacher distribution. This issue is particularly important for knowledge distillation, where the desired optimum is $q_{\bm{\theta}}=p$. Therefore, WASD adopts the Sinkhorn divergence rather than the entropy-regularized transport cost.
WASD further differs in how this objective is optimized. Although WCoRD and MCW-KD also utilize dual formulations, their roles are different: WCoRD uses a Wasserstein dual to build a contrastive representation objective, and MCW-KD uses entropic duality to formulate multi-cost heterogeneous alignment. In contrast, WASD derives a dual-potential weighted next-token objective whose stop-gradient formulation yields the gradient of the Sinkhorn divergence. The effective token weights are given by the excess optimal dual potentials, and the envelope theorem justifies computing gradients without differentiating through the Sinkhorn or fixed-point iterations. Thus, WASD differs from prior OT-based KD in the choice of token-embedding transport cost as well as in exposing the Sinkhorn divergence value gradient as a practical LLM distillation objective.

\section{Additional experimental settings}
\label{app_sec:exp_setting}

\subsection{Baselines}
\label{app_subsec:baselines}

GKD~\cite{agarwal2024onpolicy} employs the generalized Jensen-Shannon divergence with the distillation dataset and student-generated outputs. TAID~\cite{shing2025taid} proposes an adaptive assistant distribution based on the reverse KL divergence. DistiLLM~\cite{ko2024distillm} uses a skew KL divergence~\cite{lee2001effectiveness} and a skew reverse KL divergence to address optimization instability. DistiLLM-2~\cite{ko2025distillm2} proposes a contrastive distillation framework motivated by direct preference optimization~\cite{rafailov2023direct}. ABKD~\cite{wang2025abkd} employs $\alpha$-$\beta$-divergence~\cite{cichocki2011generalized} to improve the balance of probability mass allocation. CSD~\cite{kim2026distillation} uses discrete score matching to align teacher-student relative logit differences, avoiding softmax-induced smoothing. AMiD~\cite{shin2026amid} proposes a unified and generalized framework for assistant distributions and optimization in assistant-based distillation methods.

\subsection{Training and evaluation details}
\label{app_subsec:training}

All training experiments were conducted using a single NVIDIA RTX PRO 6000 GPU, while evaluation was performed on a single NVIDIA RTX 3090 GPU. Unless otherwise specified, the cost matrix for WASD was constructed using pairwise distances between the teacher model's token embeddings. We primarily employed cosine distance as the transport cost metric, while L2 distance was used for task-specific distillation experiments. As mentioned in the main text, we adopted the nearest-$k$ truncation strategy proposed by \citet{na2026semanticaware} for efficient computation, storing only the costs corresponding to the $k$ nearest neighbors and using sparse matrix multiplication. We set $k=8$ for most experiments, except for Qwen2.5, where we used $k=2$ to improve training efficiency. Unless otherwise specified in ablation studies, the entropy regularization hyperparameter was fixed to $\epsilon=0.001$, and both Sinkhorn and fixed-point iterations were run for 10 steps. The learning rate followed the default settings of each codebase: $1\times10^{-4}$ for GPT-2 and OpenLLaMA2 instruction-following experiments, $1\times10^{-5}$ for Gemma, and $5\times10^{-5}$ for Qwen. Inspired by the temperature scaling strategy used in CSD~\cite{kim2026distillation}, we additionally applied temperature scaling to the student model during training to control output diversity, typically using a temperature of 2. Without this scaling, WASD often produces more diverse outputs than the baselines at the default decoding temperature, potentially reducing ROUGE-L scores. Adjusting the decoding temperature can improve the quality-diversity trade-off, as shown in \cref{subfig:motivation_trade_off,app_fig:trade_off_main}.

The GPT-2 and OpenLLaMA2 general instruction-following experiments were conducted based on the implementation of AMiD~\cite{shin2026amid}.\footnote{\url{https://github.com/aailab-kaist/AMiD}} This implementation largely follows the codebases of MiniLLM~\cite{gu2024minillm},\footnote{\url{https://github.com/microsoft/LMOps/blob/main/minillm}} DistiLLM~\cite{ko2024distillm},\footnote{\url{https://github.com/jongwooko/distillm}} and ABKD~\cite{wang2025abkd},\footnote{\url{https://github.com/ghwang-s/abkd}} and we adopted the baseline settings provided in these implementations. 
Each setup was trained for 20 epochs, with checkpoints saved after every epoch. Final results were obtained from the checkpoint achieving the best validation ROUGE-L score. Consistent with previous studies~\cite{kim2026distillation,shin2026amid}, we evaluated each model using five evaluation seeds (10, 20, 30, 40, 50) and report the mean and standard deviation across these evaluations.
For the GPT-4 feedback evaluation in \cref{tab:gpt4_feedback}, we followed prior distillation studies~\cite{gu2024minillm,kim2026distillation,ko2024distillm}: GPT-4 assigns a score from 1 to 10 to each model response and reference answer, and we report the total model-response score divided by the total reference-answer score, multiplied by 100. For Qwen2.5 instruction-following distillation, we used the DistiLLM-2 codebase and followed the setups of DistiLLM-2 and CSD~\cite{kim2026distillation,ko2025distillm2}, distilling Qwen2.5-7B-Instruct into Qwen2.5-1.5B-Instruct. We evaluated on AlpacaEval~\cite{alpaca_eval}, Evol-Instruct~\cite{xu2024wizardlm}, and UltraFeedback~\cite{cui2024ultrafeedback} using GPT-4o-mini as the judge and report win rates against reference responses. For the controlled cost-matrix ablation in \cref{tab:transport_cost}, we followed the GPT-2 XL (1.5B) to Base (0.1B) setup in \cref{tab:main}, varying only the transport cost. The semantic cost uses cosine distances between teacher token embeddings, the permuted cost randomly reassigns token relations, and the uniform cost assigns a constant transport cost. These results are averaged over three training seeds, distinct from the evaluation seeds used above.

The Gemma-based task-specific distillation experiments followed the setup of \citet{xu2025speculative}.\footnote{\url{https://github.com/google-research/google-research/tree/master/speculative_kd}} Teacher supervised fine-tuning was conducted on the full datasets for 3 epochs for summarization and arithmetic reasoning and 10 epochs for translation. Validation was performed every 16 optimization steps, and the model checkpoint achieving the best validation performance was selected. Across all teacher training experiments, we used a batch size of 128 and a learning rate of $1\times 10^{-5}$. For KD, all baseline methods and our approach used the same checkpoint and were trained using approximately 1K samples per task. Summarization and arithmetic reasoning were distilled for 3 epochs, whereas translation was distilled for 10 epochs. Checkpoints were saved every 25 steps, and evaluation was performed using the checkpoint with the best validation loss.
The code generation experiments with Qwen2.5 followed the experimental setup of DistiLLM-2~\cite{ko2025distillm2}.\footnote{\url{https://github.com/jongwooko/distillm-2}} All methods were trained for one epoch using a fixed batch size of 32.

\section{Additional experimental results}
\label{app_sec:exp_results}

\begin{figure}[t]
    \centering
    \includegraphics[width=\linewidth]{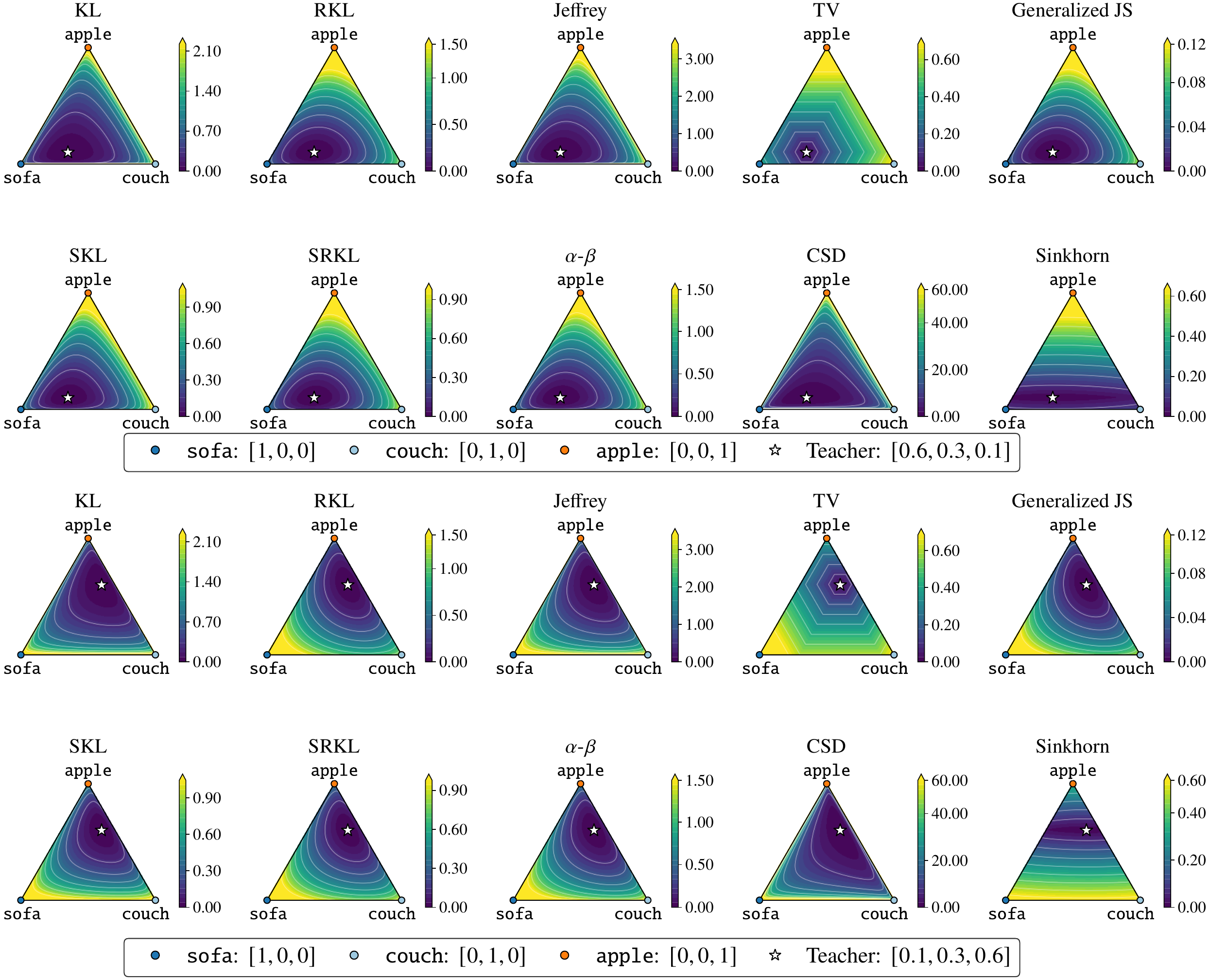}
    \caption{Loss contours of various divergences on the probability simplex for a toy next-token prediction task with different teacher distributions, extending the visualization in \cref{subfig:motivation_loss_surfaces}.}
    \label{app_fig:motivation_loss_surfaces}
\end{figure}

\begin{figure}[t]
    \centering
    \includegraphics[width=\linewidth]{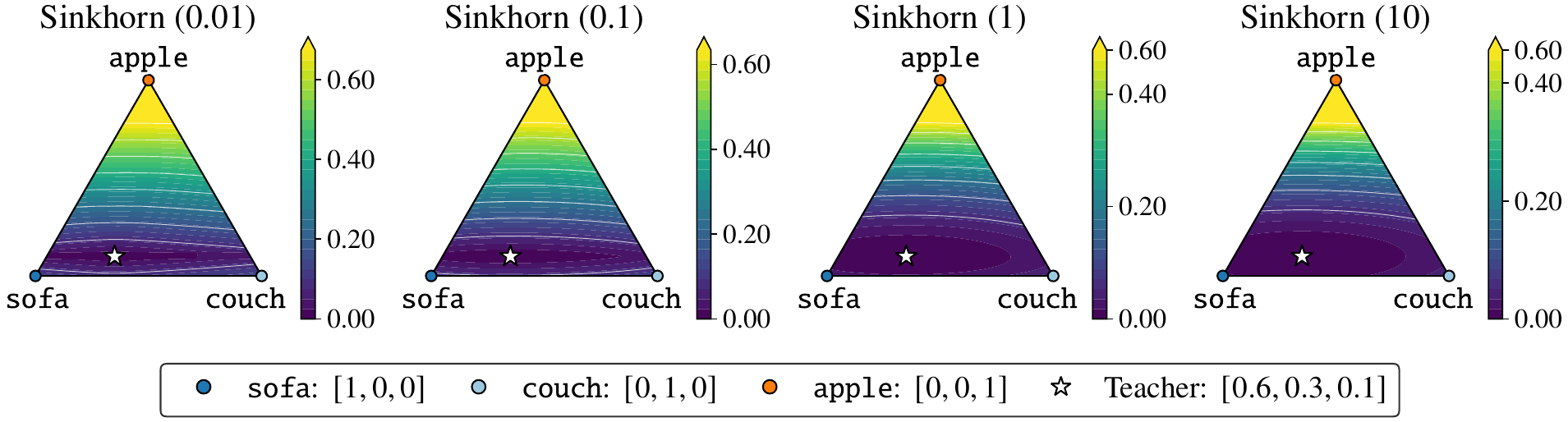}
    \caption{Loss contours of Sinkhorn divergence on the probability simplex under different entropy regularization hyperparameters $\epsilon$, where the value of $\epsilon$ is indicated in parentheses in each subplot title.}
    \label{app_fig:loss_sink_eps}
\end{figure}

\begin{figure}[tp]   
    \centering
    \begin{minipage}[t]{0.61\textwidth}
        \vspace*{0pt}
        \centering
        \includegraphics[width=0.9\linewidth]{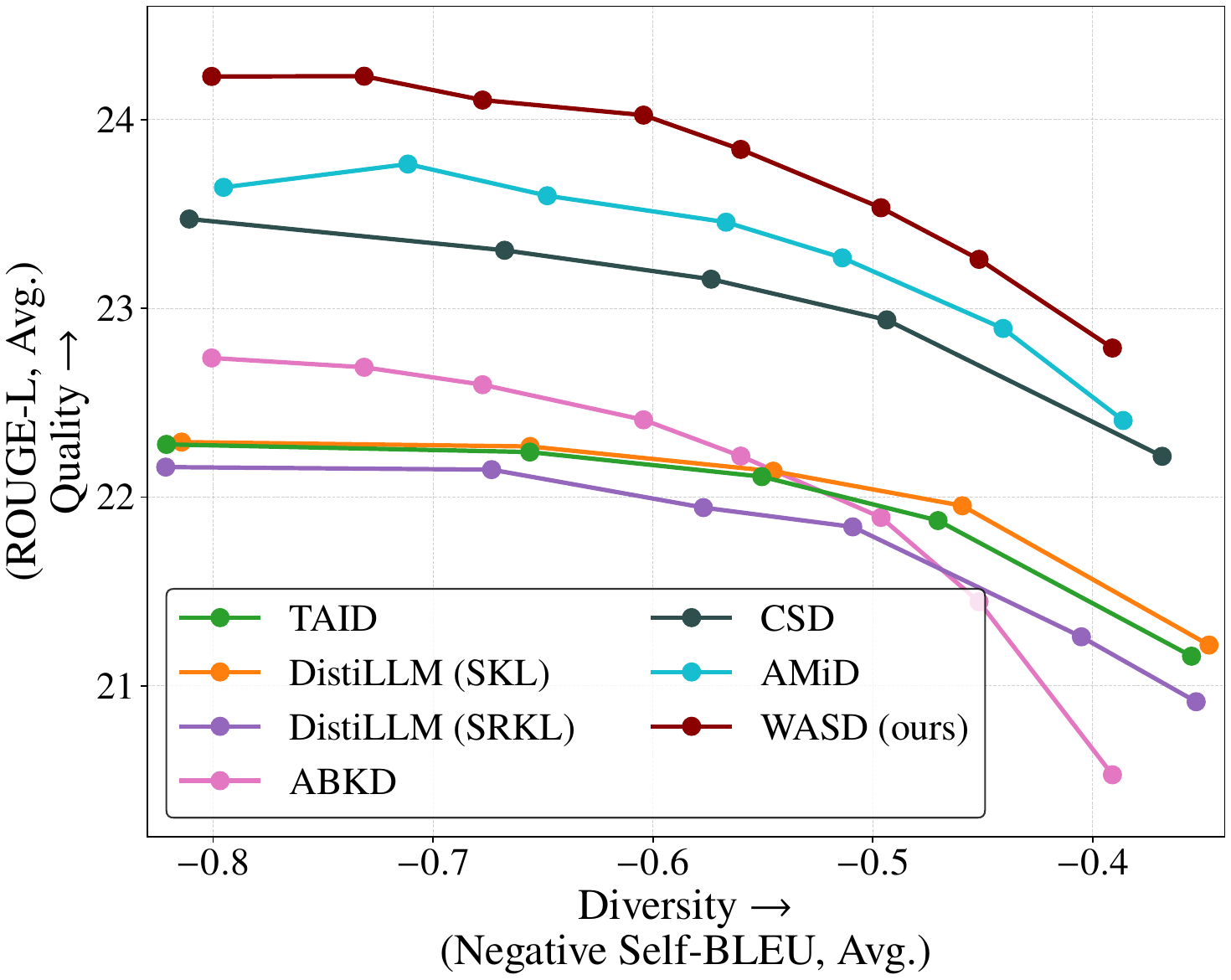}
        \caption{Quality-diversity trade-off (ROUGE-L vs. Self-BLEU) on five instruction-following benchmarks for \texttt{GPT-2} (1.5B $\rightarrow$ 0.1B), varying the decoding temperature, using the models reported in \cref{tab:main}.}
        \label{app_fig:trade_off_main}
    \end{minipage}
    \hfill
    \begin{minipage}[t]{0.35\textwidth}
        \vspace*{0pt}
        \centering
        \includegraphics[width=0.9\linewidth]{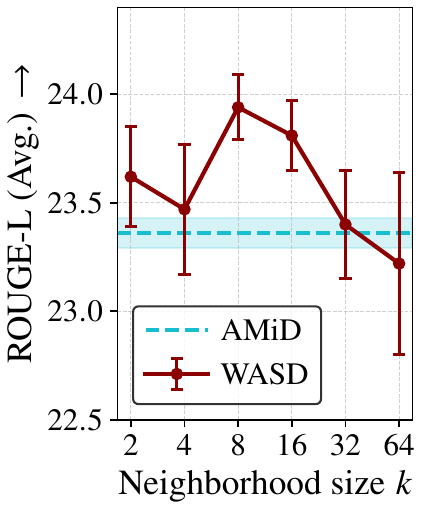}
        \caption{Sensitivity analysis of the neighborhood size $k$ for \texttt{GPT-2-1.5B} $\rightarrow$ \texttt{GPT-2-0.1B}. We report ROUGE-L across five instruction-following benchmarks over three training seeds.}
        \label{app_fig:k_sensitivity}
    \end{minipage}
\end{figure}

\begin{table}[tp]
    \centering
    \caption{Ablation study of WASD on five general instruction-following benchmarks for \texttt{GPT-2 XL (1.5B) $\rightarrow$ GPT-2 Base (0.1B)} under the same experimental setup as Table~\ref{tab:exp_divergence}.}
    \begin{tabular}{lc}
        \toprule
        Method & Avg. ROUGE-L ($\uparrow$) \\
        \midrule
        GPT-2 XL (Teacher) &  23.35 \\
        \midrule
        CSD~\cite{kim2026distillation} &  {21.46} \\
        \midrule
        \textbf{WASD} \\
        \quad Default settings & {22.47} \\
        \quad Cost change (cos $\rightarrow$ L2) &  22.45 \\
        \quad Increased truncation $k$ (8 $\rightarrow$ 16) &  22.45 \\
        \bottomrule
    \end{tabular}%
    \label{app_tab:abl}
\end{table}

\paragraph{Loss contours of divergence variants}
\cref{app_fig:motivation_loss_surfaces} presents an extended version of \cref{subfig:motivation_loss_surfaces}, visualizing loss contours on the probability simplex for a toy next-token prediction task under various divergence formulations. Across all divergences except Sinkhorn divergence, the loss landscape is primarily determined by probability assignments to vocabulary indices rather than token semantic information. Consequently, semantically similar tokens such as \texttt{sofa} and \texttt{couch} are treated as entirely distinct categories. As the teacher distribution changes, the loss contours shift according to assigned probability values instead of reflecting token semantics. This behavior consistently appears across all divergence formulations except Sinkhorn divergence.

\paragraph{Effect of entropy regularization in Sinkhorn divergence}
\cref{app_fig:loss_sink_eps} visualizes the loss contours of the Sinkhorn divergence on the probability simplex under different entropy regularization hyperparameters $\epsilon$. As $\epsilon$ increases, the landscape becomes progressively flatter, indicating that the transport kernel becomes less sensitive to semantic differences between tokens. This behavior aligns with the performance trend observed in \cref{subfig:abl_eps} of the main manuscript, where excessively large entropy regularization leads to degraded performance.

\paragraph{Quality-diversity trade-off with recent baselines}
\cref{app_fig:trade_off_main} illustrates the quality-diversity trade-off obtained by varying the decoding temperature on five instruction-following benchmarks, using the models reported in \cref{tab:main}. WASD consistently forms a favorable frontier compared to existing baselines, achieving higher ROUGE-L scores across different diversity regimes. These results suggest that the semantic-aware transport objective of WASD enables the model to preserve generation quality even under decoding settings that encourage more diverse outputs.

\paragraph{Sensitivity analysis of the neighborhood size $k$}
Following WPR~\cite{na2026semanticaware}, $k$ denotes the number of retained non-self nearest neighbors. Then, the graph is symmetrized and augmented with zero-cost self-loops, so its effective degree may exceed $k$. We conduct a sensitivity analysis over $k \in \{2,4,8,16,32,64\}$ for GPT-2 1.5B-to-0.1B distillation using three training seeds, under the same experimental setup as Table~\ref{tab:main}. As shown in \cref{app_fig:k_sensitivity}, performance does not improve monotonically with $k$: the highest mean ROUGE-L is achieved at $k=8$, while smaller neighborhoods remain competitive. Larger neighborhoods provide no consistent improvement, suggesting that including more distant token relations is not necessarily beneficial for distillation. These results support the use of small neighborhoods for efficient distillation while retaining cross-token semantic transport beyond the index-wise matching of KL divergence.

\paragraph{Ablation study on WASD}
\cref{app_tab:abl} presents an ablation study on several design choices of WASD. Replacing cosine distance with L2 distance as the transport cost yields comparable performance while still consistently outperforming the strongest baseline, CSD. Similarly, increasing the truncation parameter from $k=8$ to $k=16$ produces nearly identical results, suggesting that sparse transport with a relatively small number of nearest neighbors is sufficient for effective distillation.

\begin{table}[tp]
    \centering
    \caption{Repeated task-specific distillation experiments under the same settings as \cref{tab:task_specific}. Results are reported as mean $\pm$ standard deviation over three training seeds.}
    \begin{tabular}{lccc}
        \toprule
         & Translation & Summarization & Arithmetic \\
        Method & COMET ($\uparrow$) & ROUGE-L ($\uparrow$) & Accuracy ($\uparrow$) \\
        \midrule
        CSD~\cite{kim2026distillation} & \rateinline{73.70}{0.08} & \rateinline{34.75}{0.30} & \rateinline{23.38}{0.57} \\
        AMiD~\cite{shin2026amid} & \rateinline{73.78}{0.14} & \rateinline{34.65}{0.39} & \rateinline{24.31}{0.96} \\
        \textbf{WASD (ours)} & \rateinline{\textbf{74.40}}{0.31} & \rateinline{\textbf{34.96}}{0.18} & \rateinline{\textbf{24.62}}{0.25} \\
        \bottomrule
    \end{tabular}%
    \label{app_tab:repeat}
\end{table}

\begin{table}[tp]
    \centering
    \caption{Cross-family distillation from \texttt{Qwen1.5-1.8B} to \texttt{GPT-2-0.1B} within the DSKDv2 framework~\cite{zhang2025dual}. We report ROUGE-L with greedy decoding after three training epochs.}
    \begin{tabular}{lcccccc}
        \toprule
        Method & Dolly Eval & Self Inst & Vicuna & Super NI & UnNI & Average \\
        \midrule
        DSKDv2 (FKL) & 23.58 & 12.11 & 14.26 & \textbf{26.64} & 25.76 & 20.47 \\
        + WASD-S & 24.68 & 12.55 & 16.09 & 25.09 & 26.34 & 20.95 \\
        + WASD-S+T & \textbf{24.80} & \textbf{12.74} & \textbf{16.33} & 26.32 & \textbf{26.39} & \textbf{21.31} \\
        \bottomrule
    \end{tabular}%
    \label{app_tab:cross}
    \vspace{-1em}
\end{table}

\begin{table}[tp]
    \centering
    \caption{Diversity and semantic coherence on Dolly Eval at different decoding temperatures.}
    \begin{tabular}{ccccc}
        \toprule
        & \multicolumn{2}{c}{Distinct-2}
        & \multicolumn{2}{c}{Semantic coherence ($\uparrow$)} \\
        \cmidrule(lr){2-3}\cmidrule(lr){4-5}
        Temperature $T$ & AMiD & WASD & AMiD & WASD \\
        \midrule
        0.7 & 0.262 & 0.270 & 0.879 & \textbf{0.890} \\
        1.0 & 0.309 & 0.312 & 0.850 & \textbf{0.866} \\
        1.3 & 0.357 & 0.359 & 0.823 & \textbf{0.844} \\
        1.5 & 0.390 & 0.388 & 0.807 & \textbf{0.827} \\
        \bottomrule
    \end{tabular}
    \label{app_tab:semantic_coherence}
\end{table}

\begin{table}[tp]
    \centering
    \caption{Top next-token probabilities at a shared DialogSum context with the prefix ``Sarah is upset because she couldn't''.}
    \begin{tabular}{ll}
        \toprule
        Model & Top next tokens and probabilities \\
        \midrule
        Teacher & speak (0.19), express (0.19), share (0.19), get (0.19), voice (0.03) \\
        WASD & get (0.36), express (0.25), convince (0.16), share (0.07), finish (0.03) \\
        AMiD & get (0.63), express (0.15), convince (0.14), share (0.02), make (0.01) \\
        \bottomrule
    \end{tabular}%
    \label{app_tab:token_dist}
\end{table}

\paragraph{Repeated training results}
We repeat the translation, summarization, and arithmetic reasoning experiments in \cref{tab:task_specific} with three training seeds to examine robustness to training variation. As shown in \cref{app_tab:repeat}, WASD achieves the highest mean performance on all three tasks. For translation, WASD improves over AMiD in all three seeds, with a statistically significant difference under a paired t-test ($p=0.048$). Although the gains on the other tasks are small, the repeated-seed results show a consistent trend across the three task-specific settings.

\paragraph{Cross-model-family distillation}
To examine applicability beyond a single model family, we integrate WASD into DSKD~\cite{zhang2024dual}, a cross-family distillation framework for models with different vocabularies and tokenizers, using the updated DSKDv2 codebase~\cite{zhang2025dual}. DSKD uses cross-model projectors to express teacher and student predictions in each other's output spaces, enabling distribution matching between models with different architectures and tokenizers. We instantiate the DSKDv2 baseline with forward KL (FKL) divergence, using Qwen1.5-1.8B~\cite{bai2023qwen} as the teacher and GPT-2-0.1B as the student. WASD-S replaces the FKL-based distribution-matching term only in the student output space, whereas WASD-S+T replaces the corresponding terms in both the student and teacher output spaces. All other components remain unchanged. In each space, WASD uses a transport cost matrix constructed from the corresponding token-embedding geometry.

We use the Dolly split and evaluation benchmarks adopted by DSKD. As a controlled proof of concept, we train all compared methods for three epochs under identical optimization settings and evaluate using greedy decoding. As shown in \cref{app_tab:cross}, WASD-S improves average ROUGE-L from 20.47 to 20.95, while WASD-S+T further increases it to 21.31. WASD-S+T improves over the baseline on four of the five datasets and outperforms WASD-S on all five datasets. These results suggest that WASD can be extended to cross-family distillation through an existing cross-model projection and alignment framework. They also suggest that teacher and student vocabulary geometries provide complementary transport-based supervision.

\paragraph{Semantic coherence and token-level analysis}
We examine coherence at the response level and probability allocation at the token level.
For each Dolly evaluation prompt, we sample eight responses at decoding temperatures $T \in \{ 0.7, 1.0, 1.3, 1.5\}$. We measure lexical diversity using Distinct-2 and semantic coherence using the mean pairwise cosine similarity between response embeddings from the external sentence encoder \texttt{all-MiniLM-L6-v2}~\cite{reimers2019sentence}. We compare WASD against AMiD under the same assistant-distribution framework in the GPT-2 1.5B-to-0.1B setting. As shown in \cref{app_tab:semantic_coherence}, the two methods achieve nearly identical Distinct-2, while WASD achieves higher coherence at every temperature. The differences are statistically significant under paired Wilcoxon signed-rank tests ($p<0.01$) and generally increase as sampling diversity grows. Therefore, at comparable lexical diversity, WASD produces varied responses that remain more semantically coherent according to this metric.

At the token level, we inspect the learned next-token distributions at a shared DialogSum~\cite{chen-etal-2021-dialogsum} context where the teacher supports several plausible continuations. \cref{app_tab:token_dist} shows the top next-token probabilities given the prefix  ``Sarah is upset because she couldn't''. Although both students greedily select \texttt{get}, WASD assigns more probability to the teacher-supported alternatives \texttt{express} and \texttt{share}, whereas AMiD concentrates more probability on \texttt{get}. This example illustrates that WASD can retain a broader set of semantically related continuations supported by the teacher.

\section{Limitations}

The proposed method introduces additional computational overhead compared to standard divergence-based distillation objectives due to the Sinkhorn iterations required for solving the entropy-regularized optimal transport problem. In our experiments, WASD requires approximately twice the training time than conventional baselines and increases training GPU memory usage by approximately 20\%. Although the proposed formulation remains practically feasible for large-scale LLM distillation, improving the efficiency of Sinkhorn-based optimization remains an important direction for future work. In particular, incorporating recent advances in fast or memory-efficient Sinkhorn algorithms may further improve the practicality of semantic-aware distillation.

Another limitation is the construction of the transport cost matrix. In this work, we use distances derived from the teacher model's token embedding space as a generic measure of semantic similarity. While this simple design already provides consistent improvements, it may not fully capture task-specific or domain-specific token information. Designing more expressive or adaptive cost matrices is a promising direction for future research.

\section{Broader impacts}

This work aims to improve the efficiency of large language models through knowledge distillation, which may help reduce the computational and energy costs associated with deploying increasingly large models. More efficient distillation can improve accessibility of LLM technologies in resource-constrained environments and lower the environmental impact of inference. At the same time, advances in LLM distillation may also contribute to the wider deployment of powerful language models, which can amplify existing societal risks associated with LLMs, including misinformation generation, biased outputs, harmful content generation, and misuse in automated systems. Since the proposed method improves the ability  of smaller models to imitate stronger teacher models, these risks may become more accessible at lower computational cost. Therefore, we encourage future research on responsible deployment, safety alignment, and monitoring mechanisms alongside advances in efficient LLM compression.


\newpage
\section*{NeurIPS Paper Checklist}

\begin{enumerate}

\item {\bf Claims}
    \item[] Question: Do the main claims made in the abstract and introduction accurately reflect the paper's contributions and scope?
    \item[] Answer: \answerYes{} 
    \item[] Justification: {The abstract and introduction clearly describe the main contribution of the paper.}

\item {\bf Limitations}
    \item[] Question: Does the paper discuss the limitations of the work performed by the authors?
    \item[] Answer: \answerYes{} 
    \item[] Justification: {Limitations such as computational overhead from Sinkhorn iterations have been discussed in Section 4 and Appendix E.}

\item {\bf Theory assumptions and proofs}
    \item[] Question: For each theoretical result, does the paper provide the full set of assumptions and a complete (and correct) proof?
    \item[] Answer: \answerYes{} 
    \item[] Justification: The paper provides formal statements and proofs for the main theoretical results in Section 3, with detailed derivations in Appendix A.

    \item {\bf Experimental result reproducibility}
    \item[] Question: Does the paper fully disclose all the information needed to reproduce the main experimental results of the paper to the extent that it affects the main claims and/or conclusions of the paper (regardless of whether the code and data are provided or not)?
    \item[] Answer: \answerYes{} 
    \item[] Justification: The experimental settings, baselines, evaluation metrics, and protocols are described in Section 4 and Appendix C.

\item {\bf Open access to data and code}
    \item[] Question: Does the paper provide open access to the data and code, with sufficient instructions to faithfully reproduce the main experimental results, as described in supplemental material?
    \item[] Answer: \answerYes{} 
    \item[] Justification: The paper indicates that the implementation is publicly available and provides sufficient experimental details in \cref{app_sec:exp_setting}. All datasets and pretrained models used in the experiments are publicly available and appropriately cited.

\item {\bf Experimental setting/details}
    \item[] Question: Does the paper specify all the training and test details (e.g., data splits, hyperparameters, how they were chosen, type of optimizer) necessary to understand the results?
    \item[] Answer: \answerYes{} 
    \item[] Justification: All experimental setups, hyperparameters, datasets, and evaluation metrics are detailed in Section 4 and Appendix C.

\item {\bf Experiment statistical significance}
    \item[] Question: Does the paper report error bars suitably and correctly defined or other appropriate information about the statistical significance of the experiments?
    \item[] Answer: \answerYes{} 
    \item[] Justification: Experimental results are reported with standard deviations across five evaluation seeds in Tables 2 and 3.

\item {\bf Experiments compute resources}
    \item[] Question: For each experiment, does the paper provide sufficient information on the computer resources (type of compute workers, memory, time of execution) needed to reproduce the experiments?
    \item[] Answer: \answerYes{} 
    \item[] Justification: Details about computational resources are provided in Appendix C.
    
\item {\bf Code of ethics}
    \item[] Question: Does the research conducted in the paper conform, in every respect, with the NeurIPS Code of Ethics \url{https://neurips.cc/public/EthicsGuidelines}?
    \item[] Answer: \answerYes{} 
    \item[] Justification: The research complies with the NeurIPS Code of Ethics.

\item {\bf Broader impacts}
    \item[] Question: Does the paper discuss both potential positive societal impacts and negative societal impacts of the work performed?
    \item[] Answer: \answerYes{} 
    \item[] Justification: Potential societal impacts are discussed in Appendix F.
    
\item {\bf Safeguards}
    \item[] Question: Does the paper describe safeguards that have been put in place for responsible release of data or models that have a high risk for misuse (e.g., pre-trained language models, image generators, or scraped datasets)?
    \item[] Answer: \answerNA{} 
    \item[] Justification: The paper does not release new large-scale pretrained models or datasets that pose significant misuse risks.

\item {\bf Licenses for existing assets}
    \item[] Question: Are the creators or original owners of assets (e.g., code, data, models), used in the paper, properly credited and are the license and terms of use explicitly mentioned and properly respected?
    \item[] Answer: \answerYes{} 
    \item[] Justification: We cite the code and datasets in Section 4 and the appendix.

\item {\bf New assets}
    \item[] Question: Are new assets introduced in the paper well documented and is the documentation provided alongside the assets?
    \item[] Answer: \answerYes{} 
    \item[] Justification: The code is publicly available at \url{https://github.com/aailab-kaist/WASD}.

\item {\bf Crowdsourcing and research with human subjects}
    \item[] Question: For crowdsourcing experiments and research with human subjects, does the paper include the full text of instructions given to participants and screenshots, if applicable, as well as details about compensation (if any)? 
    \item[] Answer: \answerNA{} 
    \item[] Justification: The paper does not involve crowdsourcing or experiments with human participants.

\item {\bf Institutional review board (IRB) approvals or equivalent for research with human subjects}
    \item[] Question: Does the paper describe potential risks incurred by study participants, whether such risks were disclosed to the subjects, and whether Institutional Review Board (IRB) approvals (or an equivalent approval/review based on the requirements of your country or institution) were obtained?
    \item[] Answer: \answerNA{} 
    \item[] Justification: The paper does not involve human subjects research requiring IRB approval.

\item {\bf Declaration of LLM usage}
    \item[] Question: Does the paper describe the usage of LLMs if it is an important, original, or non-standard component of the core methods in this research? Note that if the LLM is used only for writing, editing, or formatting purposes and does \emph{not} impact the core methodology, scientific rigor, or originality of the research, declaration is not required.
    \item[] Answer: \answerYes{} 
    \item[] Justification: We use GPT-4 and GPT-4o-mini as LLM judges to evaluate generated responses. The evaluation protocols and results are described in \cref{subsec:exp_general_inst,app_subsec:training}.

\end{enumerate}

\end{document}